\documentclass[journal]{IEEEtran}

\usepackage{times}
\usepackage{epsfig}
\usepackage{graphicx}
\usepackage{subfigure}
\usepackage{caption}
\usepackage{amsmath}
\usepackage{amssymb}
\usepackage{amsthm}
\usepackage{cite}
\usepackage[noend]{algpseudocode}
\usepackage{array, bm, placeins, lipsum}
\usepackage{booktabs, color, enumerate, microtype, multirow}
\usepackage[x11names,dvipsnames,table]{xcolor} 
\usepackage{arydshln}
\usepackage{verbatim}
\usepackage{pifont}
\usepackage{engord}
\newcolumntype{x}[1]{>{\centering\arraybackslash\hspace{0pt}}p{#1}}

\usepackage{bbding}
\usepackage{wasysym}

\usepackage{array}
\newcommand{\PreserveBackslash}[1]{\let\temp=\\#1\let\\=\temp}
\newcolumntype{C}[1]{>{\PreserveBackslash\centering}p{#1}}
\newcolumntype{R}[1]{>{\PreserveBackslash\raggedleft}p{#1}}
\newcolumntype{L}[1]{>{\PreserveBackslash\raggedright}p{#1}}

\makeatletter
\newif\if@restonecol
\makeatother

\usepackage[linesnumbered,ruled,vlined]{algorithm2e}
\usepackage{algpseudocode}

\newcommand{\beq}{\vspace{0mm}\begin{equation}}
\newcommand{\eeq}{\vspace{0mm}\end{equation}}
\newcommand{\beqs}{\vspace{0mm}\begin{eqnarray}}
\newcommand{\eeqs}{\vspace{0mm}\end{eqnarray}}
\newcommand{\barr}{\begin{array}}
\newcommand{\earr}{\end{array}}

\newcommand{\xv}{\mathbf{x}}

\newcommand{\zv}{\mathbf{z}}

\theoremstyle{remark}
\newtheorem{remark}{Remark}
\newtheorem{proposition}{Proposition}
\newtheorem{corollary}{Corollary}[proposition]

\begin{document}
\title{Deep Unsupervised Image Anomaly Detection:\\ An Information Theoretic Framework}

\author{Fei~Ye,
        Huangjie~Zheng,
        Chaoqin~Huang,
        Ya Zhang
\IEEEcompsocitemizethanks{\IEEEcompsocthanksitem F. Ye, C. Huang and Y. Zhang are with the Cooperative Medianet Innovation Center and the Shanghai Key Laboratory of Multimedia Processing and Transmissions, Shanghai
Jiao Tong University, Shanghai 200240, China.
(E-mail: \{yf3310, huangchaoqin, ya\_zhang\}@sjtu.edu.cn).
H. Zheng is with Department of statistics and data science, University of Texas at Austin, Austin, 78751 TX, USA. (E-mail:  huangjie.zheng@utexas.edu).}
}

\markboth{Journal of \LaTeX\ Class Files, November~2020}%
{Shell \MakeLowercase{\textit{et al.}}: Bare Demo of IEEEtran.cls for IEEE Journals}

\maketitle

\begin{abstract}
Surrogate task based methods have recently shown great promise for unsupervised image anomaly detection. However, there is no guarantee that the surrogate tasks share the consistent optimization direction with anomaly detection. In this paper, we return to a direct objective function for anomaly detection with information theory, which maximizes the distance between normal and anomalous data in terms of the joint distribution of images and their representation. Unfortunately, this objective function is not directly optimizable under the unsupervised setting where no anomalous data is provided during training. Through mathematical analysis of the above objective function, we manage to decompose it into four components. In order to optimize in an unsupervised fashion, we show that, under the assumption that distribution of the normal and anomalous data are separable in the latent space, its lower bound can be considered as a function which weights the trade-off between mutual information and entropy. This objective function is able to explain why the surrogate task based methods are effective for anomaly detection and further point out the potential direction of improvement. Based on this object function we introduce a novel information theoretic framework for unsupervised image anomaly detection. Extensive experiments have demonstrated that the proposed framework significantly outperforms several state-of-the-arts on multiple benchmark data sets. Source code will be made publicly available.
\end{abstract}


\IEEEpeerreviewmaketitle

\section{Introduction}

\IEEEPARstart{A}{nomaly} detection, aiming to find patterns that do not conform to expected behavior, has been broadly applicable in medical diagnosis, credit card fraud detection, sensor network fault detection and numerous other fields~\cite{chandola2009anomaly}.
With the recent advance of the convolution neural network, anomaly detection for image data has received significant attention, which is challenging due to high-dimension and complex texture of images.
As anomalous data is diverse and inexhaustible, anomaly detection is usually considered under the setting where the training data contains only the ``normal'' class, \emph{i.e.}, the anomalous data is not available during training. 

In this paper, we explore unsupervised anomaly detection for images. Anomaly detection aims to learn effective feature representation that leads to the best separation of anomalous data from normal data. However, under the unsupervised setting, only the normal data is available during the model fitting. 

It seems straightforward to make  assumptions on the distribution of anomalous data and model unsupervised anomaly detection as a one-class classification problem. Several methods have explored in this direction. 
Scholkopf et al.~\cite{OC-SVM} and Oza et al.~\cite{oza2018one} assumed the anomalous distribution to be the origin and the zero centered Gaussian distribution accordingly. Ruff et al.~\cite{SVDD} assumed that the anomalous distribution to be a uniform distribution in latent space.
However, the simple assumption of the anomalous distribution can not force the network to extract effective features. Making it inevitable to rely on the pre-trained model, resulting in a two-step optimizing process.

Other than directly learning to separate anomalous and normal data, surrogate based approaches attempt to first learn feature representation with an unsupervised alternative objective function, and then assume that model will lead to poor performance to the anomalous data, which are not exposed to the training, so that they can be distinguished from the normal data. Two types of surrogate tasks are typically adopted: reconstruction~\cite{Sabokrou2018Adversarially,deecke2018anomaly,Akcay2018,OCGAN} and self-labeling~\cite{golan2018deep,NipsEffective}. For self-labeling based method, 
Golan et al.~\cite{golan2018deep} applied dozens of image geometric transforms and created a self-labeled dataset for transformation classification. Wang et al.~\cite{NipsEffective} introduced more self-label method like patch re-arranging and irregular affine transformations.
Recently, a new surrogate task: restoration~\cite{fye2020ARNet} is introduced which assumes that through restoring the erased information, the model is effectively forced to learn what is erased, and the feature embedding can thus be controlled by the corresponding information erasing. 
While surrogate approaches have become the mainstream method for anomaly detection in recent years and have shown promising results, it is hard to ensure that the surrogate tasks share the consistent optimization direction with anomaly detection. 

In this paper, we return to a direct objective function for anomaly detection, which maximizes the distance between normal and anomalous data in terms of the joint distribution for image and feature representation.
We decompose the above objective function into the following four components through mathematical analysis: \emph{Mutual information} between image space and latent space of normal data, \emph{Entropy} of normal data in latent space; \emph{Expectation of cross-entropy} between normal and anomalous samples in latent space; \emph{Distribution distance} between normal and anomalous samples in image space.
The first two components can be calculated with normal data only. The fourth component is a constant once normal data is selected.
Only optimizing the third term requires anomalous data. To optimize the objective function under the unsupervised setting, we investigate the condition to bypass the third term and get a lower bound on the objective function which can be considered as a trade-off between the mutual information and the entropy. To our best knowledge, this is the first end-to-end framework to optimize anomaly detection directly. We  provide a specific method based on the framework and further show that the lower-bound objective function can be linked to several previous studies such as the reconstruction-based method and the classic one-class classification method SVDD~\cite{SVDD}. 
As most approaches focus on only one term (mutual information or entropy), the proposed objective function can not only fill the lack of theory for many of the existing anomaly detection method but also point out the potential direction of improvement. To our best knowledge, this is the first anomaly detection objective function that can be end-to-end optimized under the unsupervised setting.

To validate the effectiveness of our lower bound objective function, we conduct extensive experiments on several benchmark datasets, including MNIST~\cite{lecun1998mnist}, Fashion-MNIST~\cite{xiao2017fashion}, CIFAR-10~\cite{krizhevsky2009learning}, CIFAR-100~\cite{krizhevsky2009learning} and ImageNet~\cite{russakovsky2015imagenet}.
Our experimental results have shown that the proposed method outperforms several state-of-the-art methods in terms of model accuracy and model stability to a large extent. The main contributions of the paper are summarized as follows:

\section{Related Works}
\subsection{Anomaly Detection}

For anomaly detection on images and videos, a large variety of methods have been developed in recent years~\cite{Chalapathy2019Deep, Markou2003Novelty_V1, Markou2003Novelty_V2, chandola2009anomaly, Pimentel2014A, Kiran2018An, chu2018sparse, xu2018anomaly, xu2019video}.  
As anomalous data is diverse and inexhaustible, anomaly detection is usually considered under the setting where the training dataset contains only ``normal'' data, i.e. the anomalous data is not available during training. The main challenges are two-fold, the first one is how to extract effective features, the second one is how to separate anomalous data in latent space.

All the previous approaches can be broadly classified into three categories, statistic based, surrogate based and one-class classification based approaches.

\textbf{Statistic based approaches}~\cite{Eskin2000Anomaly, Yamanishi2000On, Rahmani2017Coherence} assume that anomalous data will be mapped to different statistical representations that are far from the distribution of normal data. The features are typically extracted by some shallow encoders like Principal Component Analysis (PCA), kernel PCA, or Robust PCA~\cite{Xu2012Robust}.

\textbf{Surrogate based approaches} assume that the anomalous data will yield in different embedding from normal data and lead to poor performance which can be utilized as criteria to define anomaly. It manages to tackle the first challenge through unsupervised learning with an alternative objective function other than optimizing anomaly detection directly.  Three main frameworks, different in the employed supervision, are typically adopted: reconstruction-based, self-labeling-based, and outlier exposure-based frameworks.

Reconstruction-based frameworks take input image as supervision and assume that anomalous data have larger reconstruction error. The advantage is that supervision can be easily obtained. The main challenge is to find a more effective loss function to replace the typically adopted pixel-wised MSE loss, which is indicated ineffective to force the model to extract discriminate features~\cite{SimilarityMetricAutoencoding,dosovitskiy2016generating}. 
Adversarial training is leveraged to optimize the pixel-wised MSE loss, through adding a discriminator after autoencoders to judge whether its original or reconstructed image~\cite{Sabokrou2018Adversarially,deecke2018anomaly}. Akcay et al.~\cite{Akcay2018} adds an extra encoder after autoencoders and enclosing the distance between the embedding. Perera et al.~\cite{OCGAN} applied two adversarial discriminators and a classifier on a denoising autoencoder. By adding constraint and forcing each randomly drawn latent code to reconstruct examples like the normal data, it obtained high reconstruction errors for the anomalous data.

Self-label based frameworks, which take the artificial label as supervision and assume that the labels can not be predicted properly for anomalous data, has recently received significant attention. This framework, as decoder-free, can be a benefit in the advanced classification network which is proved to be more effective in extracting discriminate features. The main challenge is how to define meaningful labels.
Golan et al.~\cite{golan2018deep} applied dozens of image geometric transforms and created a self-labeled dataset for transformation classification. Wang et al.~\cite{NipsEffective} introduced a more self-label method like patch re-arranging and irregular affine transformations. 

Outlier exposure based frameworks take an auxiliary dataset entirely disjoint from test-time data as supervision and thus teach the network better representations for anomaly detection.  Hendrycks et al.~\cite{hendrycks2018deep} introduced extra data to build a multi-class classification task. The experiment revealed that even though the extra data was in limited quantities and weakly correlated to the normal data, the learned hyperplane was still effective in separating normal data.

Restoration based framework, a new framework introduced by Ye et al.~\cite{fye2020ARNet}, which assumed that through restoring the erased information the model will be effectively forced to learn what is erased and how to restore it. Thus feature embedding can be controlled by the corresponding information erasing.

\textbf{One-class classification based approaches} tackle the second challenge by making assumptions on the distribution of anomalous data, thus change the anomaly detection into a supervised binary classification problem. Explicitly, Scholkopf et al.~\cite{OC-SVM} and Oza et al.~\cite{oza2018one} assumed the anomalous distribution to be the origin and the zero centered Gaussian distribution in latent space accordingly. Implicitly, Ruff et al.~\cite{SVDD} assumed that the anomalous distribution to be a uniform distribution in latent space. Despite these approaches~\cite{OC-SVM,SVDD,oza2018one,SAD} could directly optimize the anomaly detection based objective function, the disadvantage is also obvious: these approaches have to rely on a pre-defined or pre-trained feature extractor, as its objective function and simple assumption on anomalous distribution can not force the network to extract effective feature.

\subsection{Unsupervised Learning by Maximizing Mutual Information}
Mutual information, generally employed to measure the correlation between two random variables, has recently received extensive attention in unsupervised learning. 
Mutual information is notoriously difficult to compute, particularly in continuous and high-dimensional settings~\cite{DIM}. To tackle this problem, Belghazi et al.~\cite{MINE} employ deep neural networks to effectively compute the mutual information between high dimensional input/output pairs. Hjelm et al.~\cite{DIM} utilize an 1x1 convolutional discriminator to compute the mutual information between a global summary feature vector and a collection of local feature vectors. 
Oord et al.~\cite{CPC} introduce a new NCE based loss called InfoNCE to maximize mutual information. Chen et.al~\cite{SimCLR} propose to maximize the agreement between two augmentation results from the same raw data and investigate properties of contrastive learning. Bachman et.al~\cite{AMDIM} maximize mutual information between features extracted from multiple views of a shared context. He et al.~\cite{MoCo1} utilize a dynamic dictionary to decouple the dictionary size from the mini-batch size, allowing the model to benefit from a large sampling size.

\section{Methods}
\subsection{Problem Formulation} 

Let $\mathcal{X}$ and $\mathcal{Z}$ be the domain of image data and representation data. Denote $\mathcal{X}_n$ as the normal data and $\mathcal{X}_a$ be the anomalous data, where $\mathcal{X}_n = \{\mathbf{x}_n: \mathbf{x}_n \sim \emph{p}_n(\mathbf{x})\}$, $\mathcal{X}_a = \{ \mathbf{x}_a: \mathbf{x}_a \sim \emph{p}_a(\mathbf{x})\} $. Let the corresponding representation of $\mathcal{X}_n$ and $\mathcal{X}_a$ to be $\mathcal{Z}_n$ and $\mathcal{Z}_a$ accordingly, where $\emph{p}_n(\mathbf{z})$ and $\emph{p}_a(\mathbf{z})$ are the marginal distribution of the latent representations, $\emph{p}_n(\mathbf{x}, \mathbf{z})$ and $\emph{p}_a(\mathbf{x}, \mathbf{z})$ are the joint distribution between image space and the latent space. 

To distinguish the normal data $\xv_n$ and anomalous data $\xv_a$, due to the curse of dimensionality, anomaly detection in high-dimensional image space is not favorable in general cases. 
A common choice is to define a continuous and (almost everywhere) differentiable parametric function, $E_\theta: \mathcal{X} \rightarrow \mathcal{Z}$ with parameters $\theta$ (e.g., a neural network) that encodes the data to the low-dimensional latent space $\mathcal{Z}$, where the anomalous data can be detected more easily. 
Thus a straight-forward objective function may be maximizes the conditional distributions between normal and anomalous data in latent space. We take the commonly used KL divergence as measurement metric, and objective function can be formulate as:
\begin{equation}\label{eq:KL straight-forward obj}
\max \limits_{\theta}~\operatorname{KL}\left[p_n(\mathbf{z}|\mathbf{x}) ~\|\ p_a(\mathbf{z}|\mathbf{x}) \right].
\end{equation}

In order to make the object function in Eq.~\ref{eq:KL straight-forward obj} more complete, in which the distribution distance in image space is taken into consideration simultaneously, we introduce a novel anomaly detection based objective function, which maximizes the distance between normal and anomalous data in terms of the joint distribution for image and feature representation. The objective function can be formulate as:
\begin{equation}\label{eq:KL obj}
\max \limits_{\theta}~\operatorname{KL}\left[p_n(\mathbf{x},\mathbf{z}) ~\|\ p_a(\mathbf{x},\mathbf{z}) \right].
\end{equation}

When maximizing Eq.~\ref{eq:KL obj}, all marginal distributions and all conditional distributions also match this maximization:

\begin{remark}\label{rmk:eq}
If $\operatorname{KL}\left[p_n(\mathbf{x}, \mathbf{z})|| p_a(\mathbf{x}, \mathbf{z}) \right]$ is maximized, then it is equivalent that $\operatorname{KL}\left[p_n(\mathbf{x})|| p_a(\mathbf{x}) \right]$ and $\operatorname{KL}\left[p_n(\mathbf{z}|\mathbf{x})|| p_a(\mathbf{z}|\mathbf{x}) \right]$ are maximized.
\end{remark}

\begin{proof}
The KL divergence for the joint distributions can be decomposed with chain rule \cite{cover2006elements}:
\begin{align}
\begin{split}
&\operatorname{KL}\left[p_n(\mathbf{x}, \mathbf{z})|| p_a(\mathbf{x}, \mathbf{z}) \right] 
= \mathbb{E}_{p_n(\mathbf{x}, \mathbf{z})}\left[ \log \frac{p_n(\mathbf{x}, \mathbf{z})}{p_a(\mathbf{x}, \mathbf{z})} \right]\\ 
=&\mathbb{E}_{p_n(\mathbf{x}, \mathbf{z})} \left[ \log \frac{p_n(\mathbf{x})}{p_a(\mathbf{x})} + \log \frac{p_n(\mathbf{z}|\mathbf{x})}{p_a(\mathbf{z}|\mathbf{x})} \right] \\
=& \operatorname{KL}\left[p_n(\mathbf{x})|| p_a(\mathbf{x}) \right] + \mathbb{E}_{p_n(\xv)} \left[ \operatorname{KL}\left[p_n(\zv |\mathbf{x})|| p_a(\zv |\mathbf{x}) \right] \right].
\end{split} 
\end{align}
To maximize the KL divergence for the joint distributions is equivalent to maximize the KL divergence for both marginal and conditional distributions \cite{dumoulin2016adversarially}.
\end{proof}

\vspace{-0.5cm}
\subsection{Objective Function Decomposition}

Remark \ref{rmk:eq} shows the equivalence of our objective to the common knowledge of anomaly detection. In this section we present that this objective yields a lower bound which can be optimized without anomalous data. It is a great challenge to optimize the objective function in Eq.~\ref{eq:KL obj} directly, since $p_a(\mathbf{x},\mathbf{z})$ is not accessible. To tackle this challenge,  firstly we decompose the objective function into four components, as introduced in the Proposition \ref{prop:KL-decompose}. Then, as shown in Proposition \ref{prop:obj-lb}, we investigate the condition where we can bypass the component correlated with anomalous data and thus reformulated the the objective function in Eq.~\ref{eq:KL obj} into a lower-bound in Eq.~\ref{eq:LB_KL_reformulate_2} which can be optimized only with normal data. Finally, as shown in Corollary~\ref{prop:soft-constrain}, the lower-bound can be optimized with a regularized Lagrange multiplier and get Eq.~\ref{eq:LB_KL_reformulate_2_final} as the final objective function. The Proposition \ref{prop:KL-decompose} is introduced as follows:

\begin{proposition}\label{prop:KL-decompose}
Let $I_n(\xv, \zv)$, $H_n(\zv)$, $H(p_n(\mathbf{z}|\mathbf{x}),p_a(\mathbf{z}|\mathbf{x}))$, ${KL}\left[\emph{p}_n(\xv)\ ||\ \emph{p}_a(\xv) \right]$ denote the mutual information between $\xv$ and $\zv$ for normal data, the entropy of $\zv$ for normal data, the cross entropy between $p_n(\mathbf{z}|\mathbf{x})$ and $p_a(\mathbf{z}|\mathbf{x})$ and KL divergency between $p_n(\mathbf{x})$ and $p_a(\mathbf{x})$, respectively. Note that when the dataset is given, \textit{i.e.} $p_n(\xv)$ and $p_a(\xv)$ are fixed (even though anomalous data is unknown). The objective function can be reformulated as:
\begin{align}
\begin{split}\label{eq:LB_KL_reformulate}
&\max \limits_{\theta}~  \operatorname{KL}\left[p_n(\mathbf{x}, \mathbf{z})|| p_a(\mathbf{x}, \mathbf{z}) \right] \\
=& \max \limits_{\theta}~\{\emph{I}_n\left(\mathbf{x},\mathbf{z}\right) - H_n(\mathbf{z}) +\mathbb{E}_{p_n(\mathbf{x})}\left[ H(p_n(\mathbf{z}|\mathbf{x}),p_a(\mathbf{z}|\mathbf{x}))\right]\\
& \ \ \ \ \ \ + \operatorname{KL}\left[\emph{p}_n({\xv})||\emph{p}_a({\xv}) \right] \}.
\end{split}
\end{align}
\end{proposition}
\begin{proof}
The objective function in Eq.~\ref{eq:KL obj} can be reformulated as:
\begin{align}
\begin{split}\nonumber
&\max \limits_{\theta}~\operatorname{KL}\left[p_n(\mathbf{x}, \mathbf{z})|| p_a(\mathbf{x}, \mathbf{z}) \right] \\
=&\max \limits_{\theta}~\mathbb{E}_{p_n(\mathbf{x}, \mathbf{z})}\left[ \log \frac{p_n(\mathbf{x}, \mathbf{z})}{p_a(\mathbf{x}, \mathbf{z})} \right]\\ 
=&\max \limits_{\theta}~\mathbb{E}_{p_n(\mathbf{x}, \mathbf{z})}\left[ \log \frac{p_n(\mathbf{z}|\mathbf{x}) \cdot p_n(\mathbf{x})}{p_a(\mathbf{z}|\mathbf{x}) \cdot p_a(\mathbf{x})} \right]\\
=&\max \limits_{\theta}~\mathbb{E}_{p_n(\mathbf{x}, \mathbf{z})}\left[ \log \frac{p_n(\mathbf{z}|\mathbf{x}) \cdot p_n(\mathbf{x}) \cdot p_n(\mathbf{z})}{p_a(\mathbf{z}|\mathbf{x}) \cdot p_a(\mathbf{x}) \cdot p_n(\mathbf{z})} \right]\\
=&\max \limits_{\theta}~\mathbb{E}_{p_n(\mathbf{x}, \mathbf{z})} \left[\log\left( \frac{p_n(\mathbf{z}|\mathbf{x})}{p_n(\mathbf{z})} \cdot {p_n(\mathbf{z})} \cdot \frac{1} {p_a(\mathbf{z}|\mathbf{x})} \cdot \frac{p_n(\mathbf{x})}{p_a(\mathbf{x})} \right) \right]. \\
\end{split} 
\end{align}
The above formula can be then decomposed into 4 components. 
The \engordnumber{1} component refers to the mutual information between the data sample $\xv$ and its representation $\zv$ for normal data:
\begin{align}
&\mathbb{E}_{p_n(\mathbf{x}, \mathbf{z})}\left[ \log \frac{p_n(\mathbf{z}|\mathbf{x})}{p_n(\mathbf{z})}\right] 
=\mathbb{E}_{p_n(\mathbf{x}, \mathbf{z})}\left[ \log \frac{p_n(\mathbf{z}|\mathbf{x}) \cdot p_n(\mathbf{x})}{p_n(\mathbf{x}) \cdot p_n(\mathbf{z})} \right] \nonumber \\
=&\mathbb{E}_{p_n(\mathbf{x}, \mathbf{z})}\left[ \log \frac{p_n(\mathbf{x},\mathbf{z})}{p_n(\mathbf{x}) \cdot p_n(\mathbf{z})} \right] 
= I_n\left(\mathbf{x},\mathbf{z}\right).
\end{align}
The \engordnumber{2} component is the negative entropy of $\zv$ w.r.t $p_n$:
\begin{align}
\begin{split} 
&\mathbb{E}_{p_n(\mathbf{x}, \mathbf{z})}\left[ \log {p_n(\mathbf{z})} \right] = -\mathbb{E}_{p_n(\mathbf{z})}\left[ \log \frac{1}{p_n(\mathbf{z})} \right]=-H_n(\mathbf{z}).
\end{split}
\end{align}
The \engordnumber{3} component is the expected value of the cross entropy between the conditional distributions $p_a(\zv|\xv)$ and $p_n(\zv|\xv)$:
\begin{align}
&\mathbb{E}_{p_n(\mathbf{x}, \mathbf{z})}\left[\log \frac{1}{p_a(\mathbf{z}|\mathbf{x}) }\right] 
= \mathbb{E}_{p_n(\mathbf{x})} \mathbb{E}_{p_n(\mathbf{z}|\mathbf{x})}\left[ - \log {p_a(\mathbf{z}|\mathbf{x}) }\right] \nonumber \\
=& \mathbb{E}_{p_n(\mathbf{x})}\left[ H(p_n(\mathbf{z}|\mathbf{x}),p_a(\mathbf{z}|\mathbf{x}))\right].
\end{align}
With $p_n(\xv)$ and $p_a(\xv)$ fixed, the \engordnumber{4} component is a constant:
\begin{align}
\begin{split}
\mathbb{E}_{p_n(\mathbf{x}, \mathbf{z})}\left[ \log \frac{p_n(\mathbf{x})}{p_a(\mathbf{x})} \right] = \operatorname{KL}\left[\emph{p}_n({\xv})||\emph{p}_a({\xv}) \right] = \emph{C}.
\end{split}
\end{align}
Thus the objective function in Eq.~\ref{eq:KL obj} can be reformulated as:
\begin{align}
&\max \limits_{\theta}~ \operatorname{KL}\left[p_n(\mathbf{x}, \mathbf{z})|| p_a(\mathbf{x}, \mathbf{z}) \right] \nonumber \\
=& \max \limits_{\theta}~\{\emph{I}_n\left(\mathbf{x},\mathbf{z}\right) - H_n(\mathbf{z}) +\mathbb{E}_{p_n(\mathbf{x})}\left[ H(p_n(\mathbf{z}|\mathbf{x}),p_a(\mathbf{z}|\mathbf{x}))\right] \nonumber \\
& \ \ \ \ \ \ + KL\left[\emph{p}_n({\xv})||\emph{p}_a({\xv}) \right] \} 
\end{align}
\end{proof}

The decomposed objective function in Eq.~\ref{eq:LB_KL_reformulate} is an essential general formula for optimizing anomaly detection, which combines unsupervised learning (\engordnumber{1} and \engordnumber{2} components) and supervised learning (\engordnumber{3} component). 
 
As the \engordnumber{1} and \engordnumber{2} components can be trained through an unsupervised fashion and force the encoder to extract effective features, the demand for anomalous data is greatly reduced.

To deal with unsupervised setting where anomalous data is complete absence during training, we seek to get rid of the \engordnumber{3} component, which partially relies on anomalous data.
\begin{proposition}\label{prop:obj-lb}
If $p_n(\mathbf{x}, \mathbf{z})$ and $p_a(\mathbf{x}, \mathbf{z})$ have a certain distance such that for most samples $\mathbf{x},\mathbf{z} \sim \emph{p}_n(\mathbf{x},\mathbf{z})$ the evaluated density $\emph{p}_a(\mathbf{z}|\mathbf{x})$ is small enough, \textit{such that} $\emph{p}_a(\mathbf{z}|\mathbf{x}) \leqslant \emph{p}_n(\mathbf{z})$ and $\emph{p}_a(\mathbf{z}|\mathbf{x}) \leqslant 1$ almost everywhere,

then we can derive a lower bound of Objective funcition for Eq.~\ref{eq:LB_KL_reformulate}:

\begin{align}
\begin{split}\label{eq:LB_KL_reformulate_2}
&\max \limits_{\theta}~\{\emph{I}_n\left(\mathbf{x},\mathbf{z}\right) - H_n(\mathbf{z})\}.
\end{split}
\end{align}
\end{proposition}
\begin{proof}
With the assumption that the evaluated density  $\emph{p}_a(\mathbf{z}|\mathbf{x})$ is small enough for most of the samples $\mathbf{x},\mathbf{z} \sim \emph{p}_n(\mathbf{x},\mathbf{z})$. This ensures the non-negativity of $\mathbb{E}_{p_n(\mathbf{x})}\left[ H(p_n(\mathbf{z}|\mathbf{x}),p_a(\mathbf{z}|\mathbf{x})\right]$ with the following inequality:
\begin{align}
\begin{split}
&\inf \mathbb{E}_{p_n(\mathbf{x})}\left[ H(p_n(\mathbf{z}|\mathbf{x}),p_a(\mathbf{z}|\mathbf{x})\right]\\ 
=& \inf \mathbb{E}_{p_n(\mathbf{x},\mathbf{z})} [ - \log p_a(\mathbf{z}|\mathbf{x})]  \\
\geqslant & \mathbb{E}_{p_n(\mathbf{x},\mathbf{z})} [\inf \left(-\log p_a(\mathbf{z}|\mathbf{x})\right)] \\
\geqslant &0 .
\end{split}
\end{align}
Moreover, the \engordnumber{4} component in Eq.~\ref{eq:LB_KL_reformulate} is a constant greater than zero. Then we have a lower bound to  Eq.~\ref{eq:LB_KL_reformulate}:
\begin{align}
\begin{split}
&\operatorname{KL}\left[p_n(\mathbf{x}, \mathbf{z})|| p_a(\mathbf{x}, \mathbf{z}) \right] \geqslant \emph{I}_n\left(\mathbf{x},\mathbf{z}\right) - H_n(\mathbf{z}).
\end{split}
\end{align}
\end{proof}

The objective function in Eq.~\ref{eq:LB_KL_reformulate_2} is vital as it reveals a general formula for optimizing anomaly detection in an unsupervised fashion. Note that the assumption in Proposition \ref{prop:obj-lb} is appropriate in the anomaly detection tasks. Since we often have access to data samples instead of the true data distribution, considering the empirical distribution $p_n(\zv) = \frac{1}{N}\sum_{i=1}^{N}\delta_{\zv_i}; \mathcal{Z}_n = \{\zv_i\}_{i=1}^{N} $, we always have $p_n(\zv) \leqslant 1$ and we have the evaluated density $p_a(\zv_i|\xv_i) \leqslant 1$. Moreover, with this assumption, we can futher ensure the objective can be optimized with a regularized Lagrange multiplier.

\begin{corollary}
\label{prop:soft-constrain}
With the assumption in proposition \ref{prop:soft-constrain}, the lower bound shown in Equation \ref{eq:LB_KL_reformulate_2} yields an objective function to be optimized with a entropy-regularized Lagrange multiplier:
\begin{align}
\begin{split}\label{eq:LB_KL_reformulate_2_final}
&\max \limits_{\theta}~\{\emph{I}\left(\mathbf{x},\mathbf{z}\right) - \beta \cdot \emph{H}(\mathbf{z})\},
\end{split}
\end{align}
where $\beta \geqslant 0$ is a positive coefficient for the weight of the entropy regularization.
\end{corollary}

\begin{proof}
To show Eq.~\eqref{eq:LB_KL_reformulate_2_final} is a lower bound of Eq.~\eqref{eq:LB_KL_reformulate}, we only need to show the lower bound holds when $\beta=0$:
$$\text{Eq.~\eqref{eq:LB_KL_reformulate}} - \text{Eq.~\eqref{eq:LB_KL_reformulate_2_final}} = \mathbb{E}_{p_n(\xv,\zv)} \left[ \log \frac{p_n(\zv)}{p_a(\zv|\xv)} \right] + \mathrm{KL}[p_n(\xv)||p_a(\xv)]. $$ $\mathbb{E}_{p_n(\xv,\zv)} \left[ \log \frac{p_n(\zv)}{p_a(\zv|\xv)} \right] \geq 0$ as the assumption $\frac{p_n(\zv)}{p_a(\zv|\xv)} \geq 1$ in Proposition \ref{prop:obj-lb} while the second term is a positive constant , thus $\text{Eq.~\eqref{eq:LB_KL_reformulate}} - \text{Eq.~\eqref{eq:LB_KL_reformulate_2_final}} \geq 0$ and completes the proof.
\end{proof}

\subsection{Optimization}\label{sec:opt}
In this section, we extend our general lower-bound objective to concrete loss functions. \emph{Start from here, as all the terms in Eq.~\ref{eq:LB_KL_reformulate_2} and \ref{eq:LB_KL_reformulate_2_final} refer to normal data only, we omit the subscript `n' and `a' to specify normal and anomalous.} Moreover, we present the probabilistic expressions in the empirical distribution since we work with samples from the data distribution. 

A challenge to maximize the objective function is that both mutual information and entropy are not always tractable. To maximize the mutual information between $\mathbf{x}$ and $\mathbf{z}$, we apply the Contrastive Predictive Coding lower bound with Noise Contrastive Estimation\cite{gutmann2010noise}:
\begin{equation}I(\mathbf{x} ; \mathbf{z}) \geqslant \mathbb{E}\left[\frac{1}{K} \sum_{i=1}^{K} \log \frac{e^{f\left(x_{i}, z_{i}\right)}}{\frac{1}{K} \sum_{j=1}^{K} e^{f\left(x_{i}, z_{j}\right)}}\right] \triangleq I_{\mathrm{NCE}},
\end{equation}
where $f$ is the critic function that maps the inputs into a value in $\mathbb{R}$, which can be modeled with various methods, such as a similarity function or a neural discriminator. Here, $x_i$ and $z_i = E_\theta(x_i)$ are called positive pairs; while $x_i$ and $z_j = E_\theta(x_j\mathbb{I}_{[j\not=i]})$, where $\mathbb{I}_{[j\not=i]}\in\{0,1\}$ is an indicator function evaluating to 1 iff $j\not=i$, are called negative pairs. Empirically, two independently randomly augmented versions of the same sample, \textit{e.g.} an image and its rotated view, are often used as positive pairs \cite{SimCLR,AMDIM}. 

\begin{algorithm}
  \caption{Training Pseudocode for base model}
  \label{alg::Training Pseudocode base model}  
  \KwIn{batch size N, similarity function $sim$, structure of model $E_{\theta}$, entropy weight $\beta$, constant $c_1$ represents the size of $z_i$, constant $c_2$ is set to 20 for all experiment}
  \For{sampled minibatch $\left\{x_k\right\}_{k=1}^N$}
  {
        \For{\textbf{all}$ k \in \left\{ 1,....,N\right\}$}
        {
            randomly draw two augmentation functions $t, t'\sim \mathcal{T}$\\
            \# Augmentation\\
            $\widetilde{x}_{2k-1} = t(x_k)$\\
            $\widetilde{x}_{2k} = t'(x_k)$\\
            \# Extract features\\
            $z_{2k-1} = E_\theta(\widetilde{x}_{2k-1})$\\
            $z_{2k} = E_\theta(\widetilde{x}_{2k})$\\
        }
        \For{ \textbf{all} $i \in \left\{ 1,....,2N\right\} and\; j \in \left\{ 1,....,2N\right\} $}
        {   
            \# Similarity\\
            $s_{i,j} = sim(z_i,z_j)=z^\top_i \cdot z_j$\\
            $s'_{i,j} = c_2\cdot tanh\left(\frac{s_{i,j}}{c_1\cdot c_2} \right)$ \\
        }
        
        \textbf{define} $S = \left\{ s'_{i,j} ,i,j \in {2N} \right\},\; s_{max} = \max\left\{S\right\}$
        
        $S_{shift} = S-s_{max}$\\
        \textbf{define} $S_{shift} = \left\{ \hat{s}_{i,j} ,i,j \in {2N} \right\}$
        
        \textbf{define} $\ell(i,j) =-\log \frac{exp\left(\hat{s}_{i,j}\right)}{ \sum_{k=1}^{2N}\mathbb{I}_{[k\not=i]} exp(\hat{s}_{i,k})} $\\
        $\mathcal{L}_{nce}=\frac{1}{2N}\sum_{k=1}^N\left[ \ell(2k-1,2k)+\ell(2k,2k-1)\right]$\\
        $\mathcal{L}_{entropy}=\frac{1}{2N}\sum_{k=1}^{2N} \|z_{k}\|_q$\\
        $\mathcal{L}=\mathcal{L}_{nce} + \beta \cdot \mathcal{L}_{entropy}$\\
        update network $E_{\theta}$
    }
    return encoder network $E_\theta$ 
\end{algorithm}

To minimize the entropy, we seek the lower bound for the negative entropy as:
 \begin{align}
 \label{eq:obj_radius}
 - H(\mathbf{z}) &= \mathbb{E}_{p(\mathbf{x}, \mathbf{z})} [\log  p(\mathbf{z})] \nonumber \\
 &= \mathbb{E}_{p(\mathbf{x}, \mathbf{z})} [\log  p(\mathbf{z}) - \log  r(\mathbf{z}) ] + \mathbb{E}_{p(\mathbf{x}, \mathbf{z})} [\log  r(\mathbf{z})] \nonumber \\
 &= \operatorname{KL} [ p_n(\mathbf{z}) || r(\mathbf{z}) ] + \mathbb{E}_{p(\mathbf{x}, \mathbf{z})} [\log  r(\mathbf{z})] \nonumber  \\
 &\geqslant  \mathbb{E}_{p(\mathbf{x}, \mathbf{z})} [\log  r(\mathbf{z})], 
 \end{align}
where $r(\mathbf{z})$ is a reference distribution and $\operatorname{KL} [ p_n(\mathbf{z}) || r(\mathbf{z})]\geqslant0$. One common choice is to let $r(\mathbf{z}) = \mathcal{N}(\mathbf{0}, I)$. Then, this lower bound is proportional to the L2 norm:
\begin{equation}\label{eq:obj_activate}
\mathbb{E}_{p(\mathbf{x}, \mathbf{z})} [\log r(\mathbf{z}) ] \propto - \frac{1}{N} \sum_{i=1}^N \| z_i \|_2^2,
\end{equation}
where $\| \cdot \|_2$ denotes the Frobenius norm when $p=2$ and $z_i = \emph{E}_\theta(x_i)$ denotes the latent feature of the $i^{th}$ normal sample. Eq.~\ref{eq:obj_activate} also provides us with a geometrical interpretation in the latent space. In this view, we can also generalize the Frobenius case into other orders, such as $p=1$: Another choice is thus to let $r(\mathbf{z})$ as a standard \textit{Laplace} distribution ($r(\mathbf{z}) = \operatorname{L}(\mathbf{0},I)$), similarly, this yields the mean of L1 norm:
\begin{equation}\label{eq:obj_activate_1}
\mathbb{E}_{p(\mathbf{x}, \mathbf{z})} [\log r(\mathbf{z}) ] \propto - \frac{1}{N} \sum_{i=1}^N \| z_i \|_1.
\end{equation}
Based on Eq. \ref{eq:LB_KL_reformulate_2_final}, with Monte Carlo samples, our loss function is defined as:
\begin{align}
\begin{split}\label{eq:loss_prime}
\mathcal{L} &= - I_{\operatorname{NCE}}(\mathbf{x}, \mathbf{z}) - \beta \cdot \mathbb{E}_{p(\mathbf{x}, \mathbf{z})} [\log  r(\mathbf{z})]  \\
&\approx \frac{1}{N} \sum_{i = 1}^N \left[ - \log \frac{e^{f\left(x_i, z_i\right)}}{\frac{1}{K} \sum_{j=1}^{K} e^{f\left(x_i, z_j\right)}} + \beta \cdot \| z_i \|_p \right],\\
\end{split}
\end{align}
where $p=1$ or $p=2$ according to the choice of $r$. We will explain the difference of this choice in the experiment part.

In this paper, we apply $f(x_i,z_i)= sim(E_\theta(x_i), E_\theta(\widetilde{x}_i)) = sim(z_i, \widetilde{z}_i)$, where $z_i$ and $\widetilde{z}_i$ are the latent representations of $x_i$ and its augmented view $\widetilde{x}_i$; $sim(u,v)\equiv u^\top \cdot v $ denotes the similarity between two normalized vectors u and v.
The loss function can be reformulated as:
\begin{align}
\begin{split}\label{eq:loss}
\mathcal{L} 
&=\frac{1}{K} \sum_{i = 1}^K \left[ - \log \frac{exp\left(sim\left(z_i, \widetilde{z}_i\right)\right)}{\frac{1}{K} \sum_{j=1}^{K} exp(sim(z_i, \widetilde{z}_j))} + \beta\| z_i \|_p \right],
\end{split}
\end{align}

Note that in practice the NCE lower bound requires larger number of negative samples to ensure the good performance \cite{SimCLR,poole2019variational}. Considering the efficiency of negative sampling in anomaly detection, we follow SimCLR\cite{SimCLR} and AMDIM\cite{AMDIM} for the augmentation strategy.
The pseudocode of the training process for the base model is shown in Algorithm \ref{alg::Training Pseudocode base model}.

Local Deep Infomax (local DIM)~\cite{DIM}, which maximizes the mutual information between a global feature vector depend on the full input and a collection of local feature vectors pulled from an intermediate layer in the encoder, has shown to be greatly effective in improving feature learning and maximizing mutual information in~\cite{DIM,AMDIM}. 
To get the best performance, we introduce local DIM to our method as the extension model. With the augmentation strategy from \cite{SimCLR,AMDIM}, the negative samples size is increased and $L_{NCE}\left(z_i,\widetilde{z}_i \right)$ becomes:
\begin{align}
\begin{split}\label{eq:L_NCE}
\mathcal{L} _{NCE}\left(z_i,\widetilde{z}_i \right)
&= - \log \frac{exp\left(sim\left(z_i, \widetilde{z}_i\right)\right)}{\frac{1}{2K} \sum_{j=1}^{2K}\mathbb{I}_{[j\not=i]} exp(sim(z_i, \widetilde{z}_j))}.
\end{split}
\end{align}

Thus the loss function, added local DIM, is formulated as:
\begin{align}
\begin{split}\label{eq:loss_final}
&\mathcal{L}=\frac{1}{K} \sum_{i = 1}^K \left[\mathcal{L}_{NCE}\left(z_{g_i},\widetilde{z}_{g_i} \right) + \mathcal{L}_{NCE}\left(z_{g_i},\widetilde{z}_{l_i} \right)+\beta \cdot \|z_{g_i}\|_p \right],
\end{split}
\end{align}
where $z_{g_i}$ refers to the global features, produced by parametric function $E_\theta$, $z_{l_i}$ refers to the local features, produced by an intermediate layer in $E_\theta$.

\subsection{Normal Score Measurement}
Most surrogate supervision based approaches associate the anomaly score measurement with the loss function of surrogate task, which  assume that anomalous data will result in relatively bigger loss. Following this mechanism, since the similarity between $\mathbf{z}_i$ and $\widetilde{\mathbf{z}}_i$ is maximized in our method, a straight forward normal score measurement can be formulated as:
\begin{align}
\begin{split}\label{eq:anomaly score}
{\operatorname{NormalScore}}=sim(z_i, \widetilde{z}_i),
\end{split}
\end{align}
during $z_i$ and $\widetilde{z}_i$ are the latent representations of the augmented view $x_i$ and $\widetilde{x}_i$ of the same test data example $x_{ori_i}$, in which augmented view is randomly selected. In this design, sample with high normal score is consider to be normal. However, it is time consuming to traverse all combinations of augmented view pair and random selected one combination yield in unstable result. As a solution, when testing we skip the augmentation terms and encode the original data $\mathbf{x_{ori_i}}$ directly. The corresponding output features is noted as $\mathbf{z_{ori_i}}$. Thus we reformulated the normal score for base model as:
\begin{align}
\begin{split}\label{eq:anomaly score final}
&{\operatorname{NormalScore}}=sim(z_{ori_i},z_{ori_i}).
\end{split}
\end{align}
The normal score for extension model is then reformulated as:
\begin{align}
\begin{split}\label{eq:anomaly score final extention}
&{\operatorname{NormalScore}}=sim(z_{g_i},z_{g_i})+sim(z_{g_i},z_{l_i}).
\end{split}
\end{align}

\subsection{Relation to Other Algorithms}

A successful theorem should not only be able to explain the implicit mechanisms in previous works but also be able to point out the potential improvement direction. We then take some classic methods to illustrate the interpretability of work.

\textbf{AutoEncoder~\cite{vincent2008extracting}:} 
Reconstruction based anomaly detection using Auto-encoder is a mainstream method. It assumes that by minimizing the reconstruction error, normal and anomalous samples could lead to significantly different embedding and thus the corresponding reconstruction errors can be leveraged to differentiate the two types of samples. Then we will see how our equation fits this work.
First we reformulate the mutual information terms as:
\begin{align}
\begin{split}\label{eq:MI_AE}
&I_n\left(\xv,\zv\right) = H_n(\xv) - H_n(\xv|\zv)\\
=& H_n(\xv) + \mathbb{E}_{\mathbf{x}\sim p_n(\mathbf{x})}\mathbb{E}_{\mathbf{z}\sim p_n(\mathbf{z}|\mathbf{x})}\left[ \log{p_n(\mathbf{x}|\mathbf{z})} \right]. 
\end{split}
\end{align}

With Eq.~\ref{eq:MI_AE}, the Eq.~\ref{eq:LB_KL_reformulate_2} is reformulated as:
\begin{align}
\begin{split}\label{eq:AE_loss}
&\max \limits_{\theta}~\{\emph{I}_n\left(\mathbf{x},\mathbf{z}\right) - H_n(\mathbf{z})\}\\
=& \max \limits_{\theta}~\{H_n(\xv) + \mathbb{E}_{\mathbf{x}\sim p_n(\mathbf{x})}\mathbb{E}_{\mathbf{z}\sim p_n(\mathbf{z}|\mathbf{x})}\left[ \log{p_n(\mathbf{x}|\mathbf{z})} \right]-H_n(\zv)\}, 
\end{split}
\end{align}
where the first term is a constant, the second term is the reconstruction likelihood, the third term is the entropy of $\mathcal{Z}$. This gives solid mathematical support to why anomaly detection can be benefited via maximizing reconstruction likelihood. Furthermore, it is indicated that adding entropy terms may further improve the method.

\textbf{SVDD\cite{SVDD}:} 
Deep one-class classification is the most classic method. Pre-trained by autoencoder network with reconstruction error, it then minimizes the volume of a data-enclosing hyper-sphere in latent space. It assumes that the normal examples of the data are closely mapped to center c, whereas anomalous examples are mapped further away from the center. SAD\cite{SAD} mathematically proved that the objective function in SVDD minimizes an upper bound on the entropy of latent space. Considering the pre-trained autoencoding objective implicitly maximizes the mutual information, SAD intuitively summarizes its objective function to have a positive correlation with mutual information and negative correlation with entropy, 

which is consistent with our objective function in Eq.~\ref{eq:LB_KL_reformulate_2}. our method gives theoretical support to why anomaly detection can benefit from this objective function. Both SAD and SVDD used a two-stage optimization method, which maximizes the mutual information first and then minimizes the entropy. 

This two-stage optimization can not guarantee the simultaneous optimization of the mutual information and the entropy. Moreover, to relief the hyper-sphere collapse, only networks without bias terms and bounded activation functions can be used~\cite{SVDD}.

\textbf{Information bottleneck \cite{tishby2000information,tishby2015deep}:}
The Information Bottleneck (IB) methods define a good representation and learn it with the trade-off of  a concise representation and  a powerful prediction in downstream tasks \cite{tishby2000information,tishby2015deep}. Various studies, extend the IB methods into deep learning scenario with variational methods, such as Variational Information Bottleneck (VIB) \cite{VIB}, Information Confidence Penalty (ICP) \cite{ICP}, Variational Discriminator Bottleneck (VDB) \cite{VDB}, \textit{etc.} Our regularizer shown in Eq.~\ref{eq:LB_KL_reformulate_2} is consistent with the one proposed in \cite{ICP}. We also show in Eq.~\ref{eq:obj_radius} that the entropy is proportional to a KL divergence in terms of a reference distribution, which is consistent with the regularizer used in \cite{VIB,VDB}.

\textbf{Basic assumption for surrogate task based approaches in anomaly detection:} As anomalous data is inaccessible during train process, most surrogate tasks based unsupervised anomaly detection methods are based on an assumption that data can be embedded into a lower-dimensional subspace where normal and anomalous samples appear significantly different~\cite{chandola2009anomaly}. In our method, we are the first to reveal the theoretical rationality for this basic assumption. From Proposition \ref{prop:obj-lb}, we can see the key to reformulate the objective function from semi-supervised fashion (objective function in Eq.~\ref{eq:LB_KL_reformulate}) to unsupervised fashion (objective function in Eq.~\ref{eq:LB_KL_reformulate_2}) is to ensures the non-negativity of $\mathbb{E}_{p_n(\mathbf{x})}\left[ H(p_n(\mathbf{z}|\mathbf{x}),p_a(\mathbf{z}|\mathbf{x})\right]$. This is under the assumption, as introduced in Proposition \ref{prop:obj-lb}, that $p_n(\mathbf{x}, \mathbf{z})$ and $p_a(\mathbf{x}, \mathbf{z})$ has certain distance such that for most samples $\mathbf{x},\mathbf{z} \sim \emph{p}_n(\mathbf{x},\mathbf{z})$ the evaluated density  $\log \emph{p}_a(\mathbf{z}|\mathbf{x}) \leqslant 0$, which is consistant with the assumption in~\cite{chandola2009anomaly}.

\section{Experiments}
In this section, we present a comprehensive set of experiments to validate our anomaly detection algorithm under unsupervised settings, in which multiple commonly used benchmark datasets are involved. To further demonstrate the robustness of our method, following the setting of ARNet\cite{fye2020ARNet}, a subset of ImageNet~\cite{russakovsky2015imagenet} with higher resolution, richer texture and more complex background, is utilized.

\subsection{Experimental Setups}

\textbf{Datasets.}
We experiment with the following five popular image datasets. 
\begin{itemize}
    \item
    MNIST~\cite{lecun1998mnist}: consists of 70,000 $28\times28$ handwritten grayscale digit images.
    \item
    Fashion-MNIST~\cite{xiao2017fashion}: a relatively new dataset comprising $28\times28$ grayscale images of 70,000 fashion products from 10 categories, with 7,000 images per category. 
    \item CIFAR-10~\cite{krizhevsky2009learning}: consists of 60,000 $32\times32$ RGB images of 10 classes, with 6,000 images for per class. There are 50,000 training images and 10,000 test images, divided in a uniform proportion across all classes. 
    \item
    CIFAR-100~\cite{krizhevsky2009learning}: consists of 100 classes, each of which contains 600 RGB images. The 100 classes in the CIFAR-100 are grouped into 20 ``superclasses'' to make the experiment more concise and data volume of each selected ``normal class'' larger. 
    \item
    ImageNet~\cite{russakovsky2015imagenet}: Following ARNet~\cite{fye2020ARNet}, we group the data from the ILSVRC 2012 classification dataset~\cite{russakovsky2015imagenet} into 10 superclasses by merging similar category labels using Latent Dirichlet Allocation (LDA)~\cite{blei2003latent}, a natural language processing method (see appendix for more details). We noted this dataset is more challenging due to its higher resolution richer contexture and more complex background. 
\end{itemize}
For all datasets, the training and test partitions remain as default. In addition, pixel values of all images are normalized to $[-1, 1]$.

\textbf{Evaluation protocols.} 
For a dataset with $C$ classes, we conduct a batch of $C$ experiments respectively with each of the $C$ classes set as the ``normal'' class once. We then evaluate performance on an independent test set, which contains samples from all classes, including normal and anomalous data. As all classes have equal volumes of samples in our selected datasets, the overall number proportion of normal and anomalous samples is $1:C-1$. The model performance is then quantified using the area under the Receiver Operating Characteristic (ROC) curve metric (AUROC). It is commonly adopted as performance measurement in anomaly detection tasks and eliminates the subjective decision of threshold value to divide the ``normal'' samples from the anomalous ones. The above evaluation protocols are generally accepted among most recent works on anomaly detection \cite{kingma2013auto, zhai2016deep, zong2018deep,schlegl2017unsupervised,Akcay2018,OCGAN,golan2018deep,deecke2018anomaly,fye2020ARNet}.

\textbf{Model configuration.} 

For all the datasets, the training epoch is set as 400, the learning rate is set as 2e-4. 
The corresponding pseudocode of base model can be found in Algorithm~\ref{alg::Training Pseudocode base model}.

For all experiments, if not specified, the base model refers to the model with loss function in Eq.~\ref{eq:LB_KL_reformulate_2_final}, normal scoring function in Eq.~\ref{eq:anomaly score final}, pseudocode in Algorithm \ref{alg::Training Pseudocode base model}. The extensive model refers to the model with loss function in Eq.~\ref{eq:loss_final}, normal scoring function in Eq.~\ref{eq:anomaly score final extention}. In general, we use a base model to investigate the property of the objective function and use an extensive model to reach better performance.

\subsection{Investigation of Objective Function Properties}\label{sec:Property}
In this section, we look into the properties of our objective function with our base model. 
The objective function in Eq.~\ref{eq:LB_KL_reformulate_2_final} provides a lower bound for the maximization of the KL divergence in Eq.~\ref{eq:KL obj}, which present a trade-off between the maximization of the mutual information $I(\xv,\zv)$ and the constrain of the entropy $H(\zv)$.

\subsubsection{\textbf{Hyper parameter $\beta$ in adjusting the trade-off}}\label{sec:balance}
To investigate the trade-off between the mutual information and the entropy, several experiments are conducted with a set of $\beta$ , noted as $\mathcal{B}=\{0,0.5,1,10,20,30,40,50,60\}$, on datasets CIFAR10 and CIFAR100. We report the average AUROC results of all classes with respect to $\beta$ ( Fig.~\ref{img:balence}).
We observe that the model yield a comparable good performance with $\beta$ in a relatively wide range from $0.5$ to $40.0$ in CIFAR10 and from $1$ to $30.0$ in CIFAR100, where the best performance is reached when $\beta = 20.0$ on both datasets, suggesting that the model is not very sensitive to the change in $\beta$ as long as it in a certain range, which is vital in real-world cases where we can not utilize a testing dataset to select the super-parameter. In contrast, the model results in poor performance when $\beta$ is either too small or too large.

In this work, if not specified, we choose $\beta=20.0$ to conduct experiments on all datasets. 

We also record the AUROC result every 5 epochs during the training process on CIFAR10 with different $\beta$ in $\mathcal{B}$. The result of class 1 (car) is presented in Fig.~\ref{img:Entropy_WITH_OR_WITHOUT}. It is observed that the AUROC curve shows better convergence with $\beta$ in the range from $0.5$ to $30.0$, where the model from the last epoch reaches almost the best performance. This is a vital property in real-world cases, where we can not utilize testing results to choose the model with the best performance. 

\begin{figure}[t]
\centering
\includegraphics[width=\columnwidth]{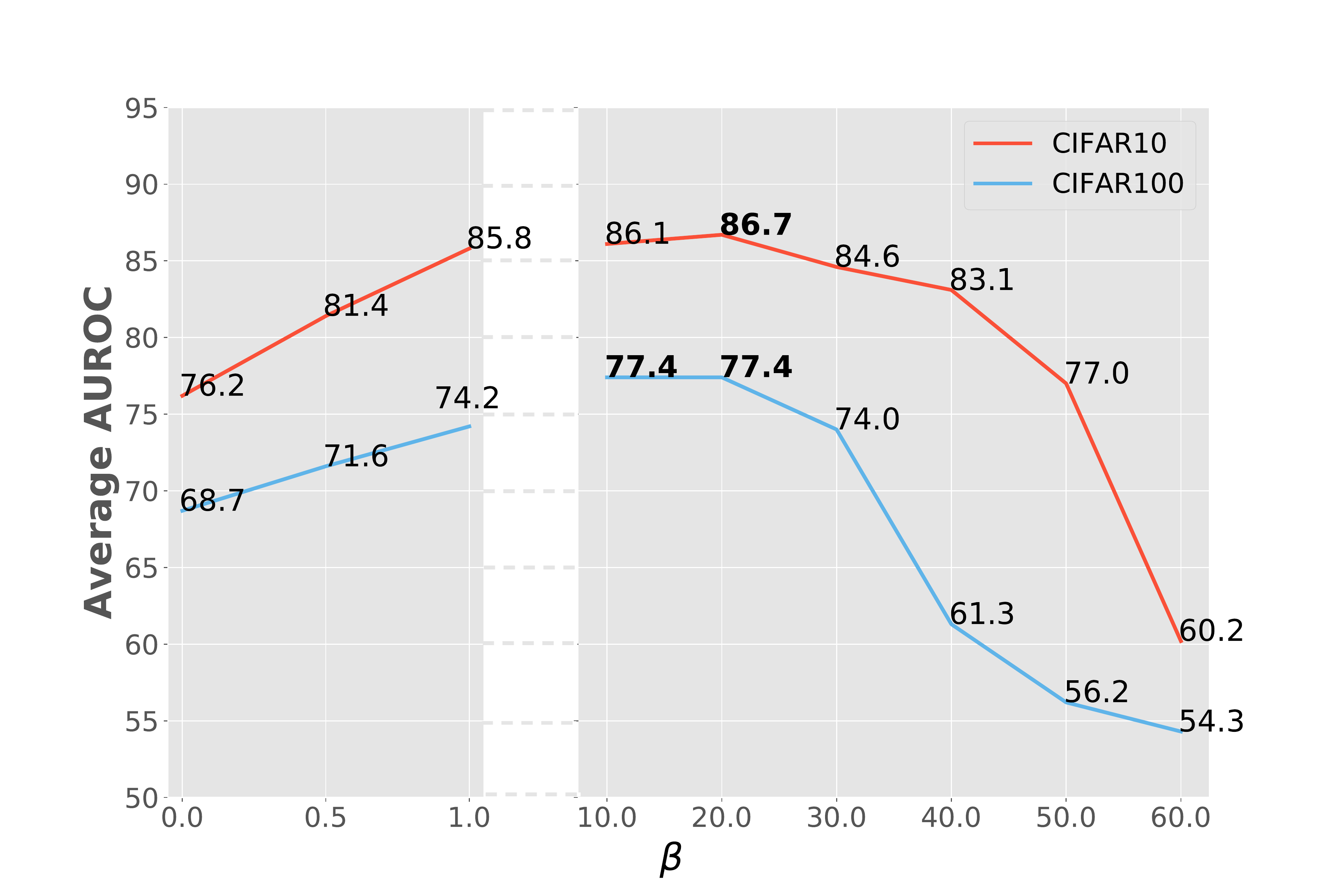}
\caption{Average AUROC \textit{w.r.t.} $\beta$ on CIFAR10 and CIFAR100. }
\label{img:balence}
\end{figure}

\begin{figure}[t]
\centering
\includegraphics[width=\columnwidth]{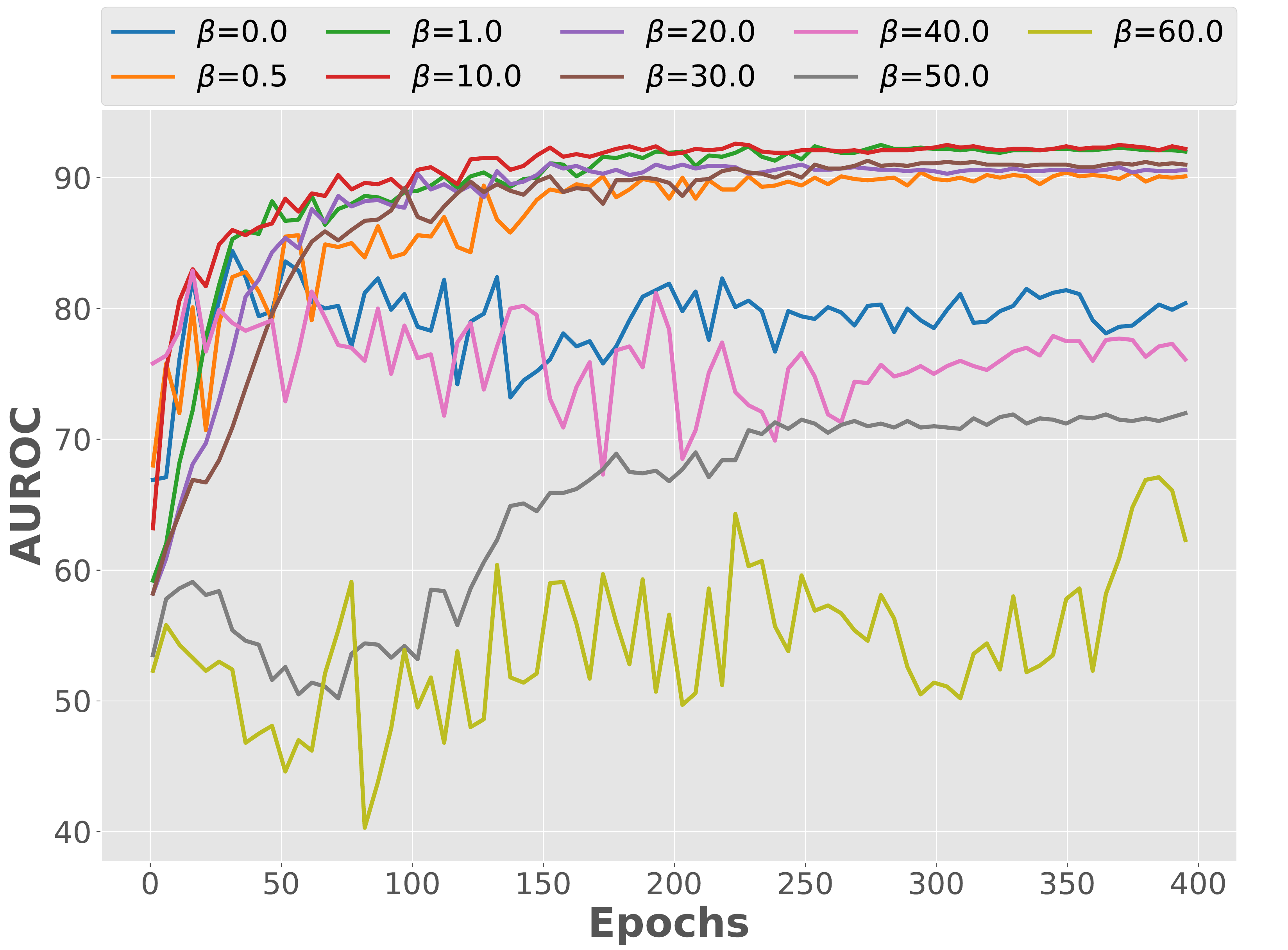}
\caption{The testing AUROC \textit{w.r.t.} $\beta$ during training. }
\label{img:Entropy_WITH_OR_WITHOUT}
\end{figure}

To be noted, as observed from the above two experiments, when $\beta=0.0$ the model can not converge properly, which indicates that without the entropy constrain the unsupervised learning framework based on maximizing mutual information only is not directly applicable to anomaly detection.

\begin{figure*}[!ht]
\centering
\begin{minipage}[t]{0.49\textwidth}
\centering
\includegraphics[width=\textwidth]{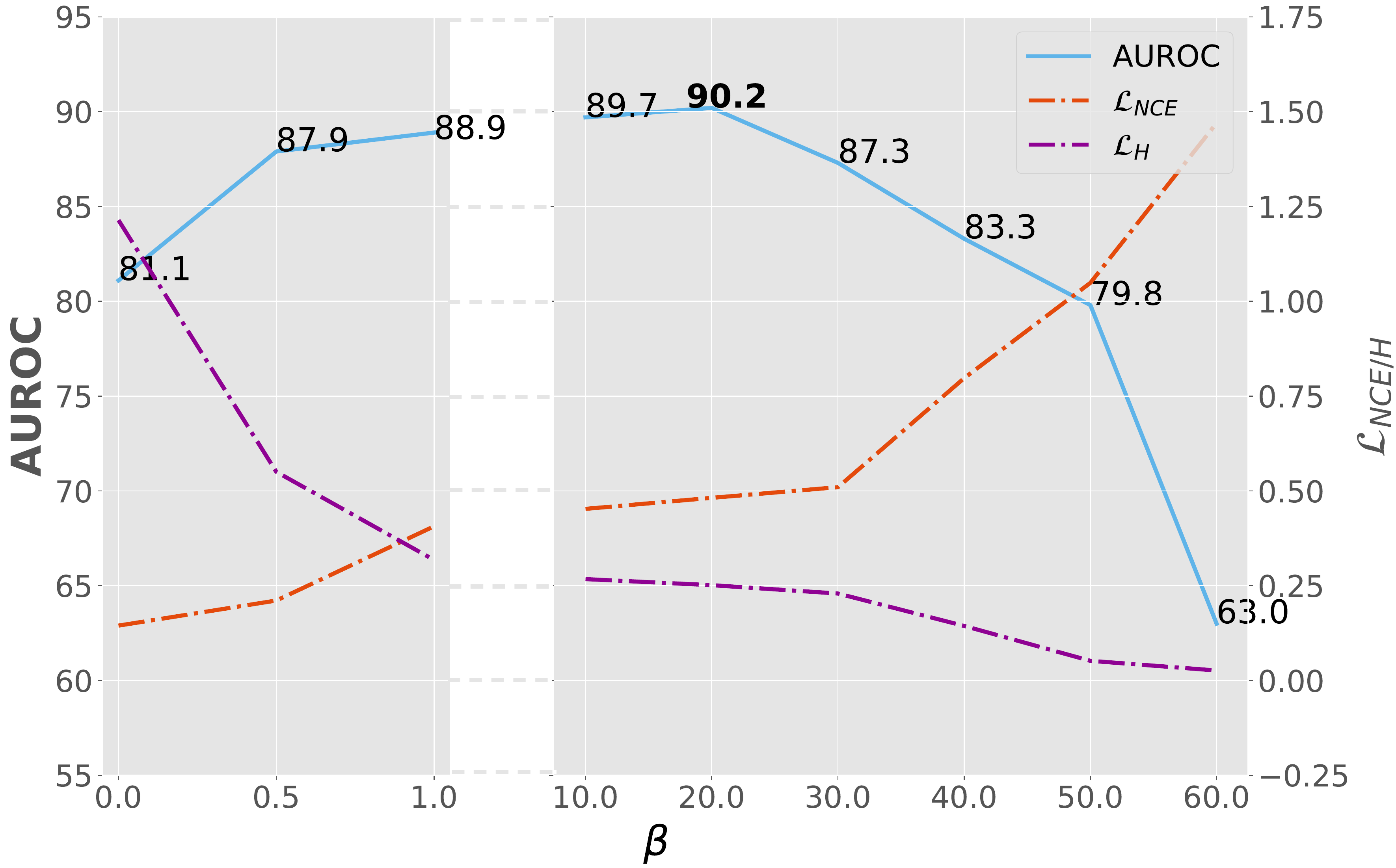}
\end{minipage}
\hfill
\begin{minipage}[t]{0.47\textwidth}
\centering
\includegraphics[width=\textwidth]{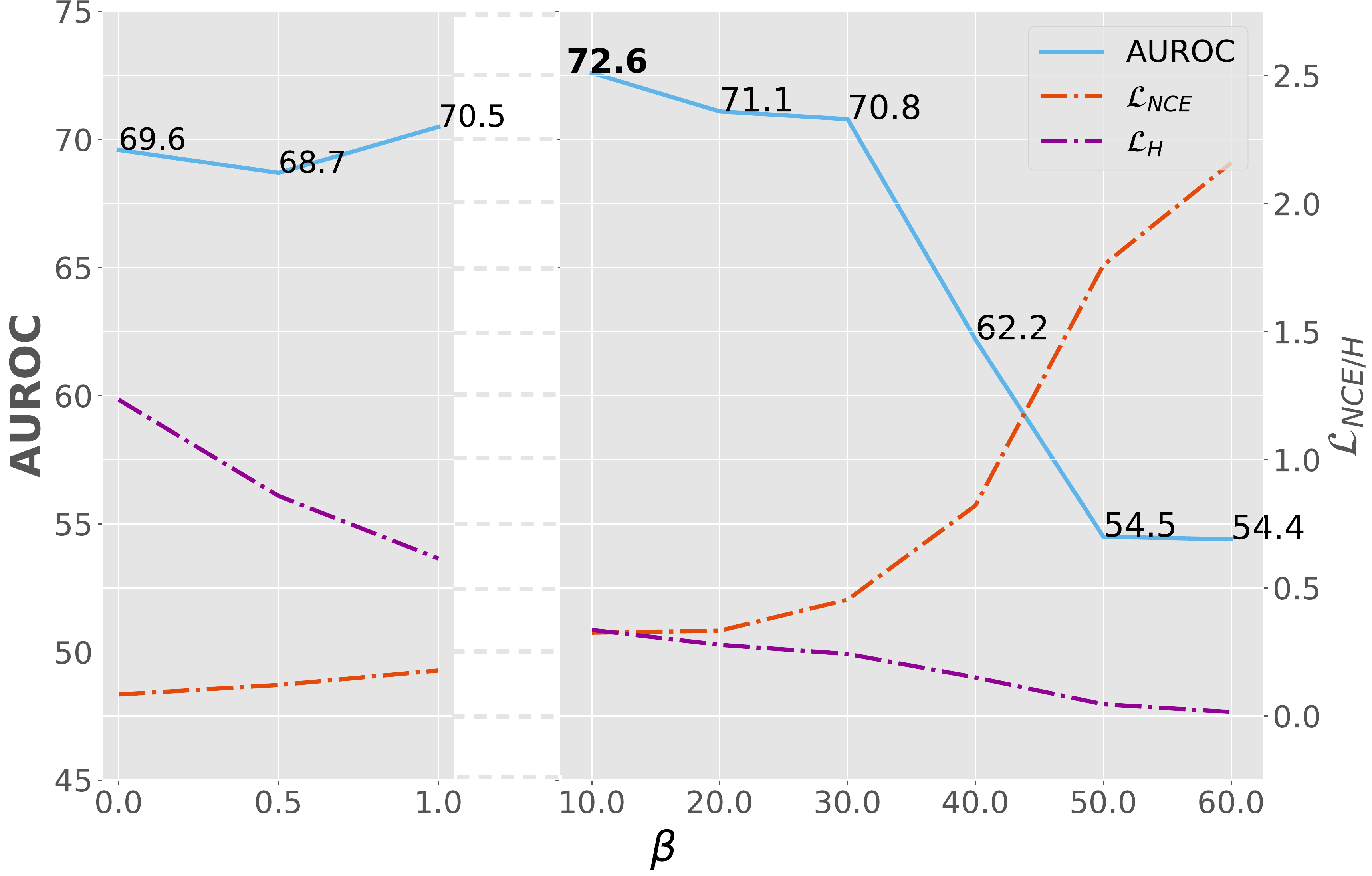}
\end{minipage}
\begin{minipage}[t]{0.49\textwidth}
\centering
\includegraphics[width=\textwidth]{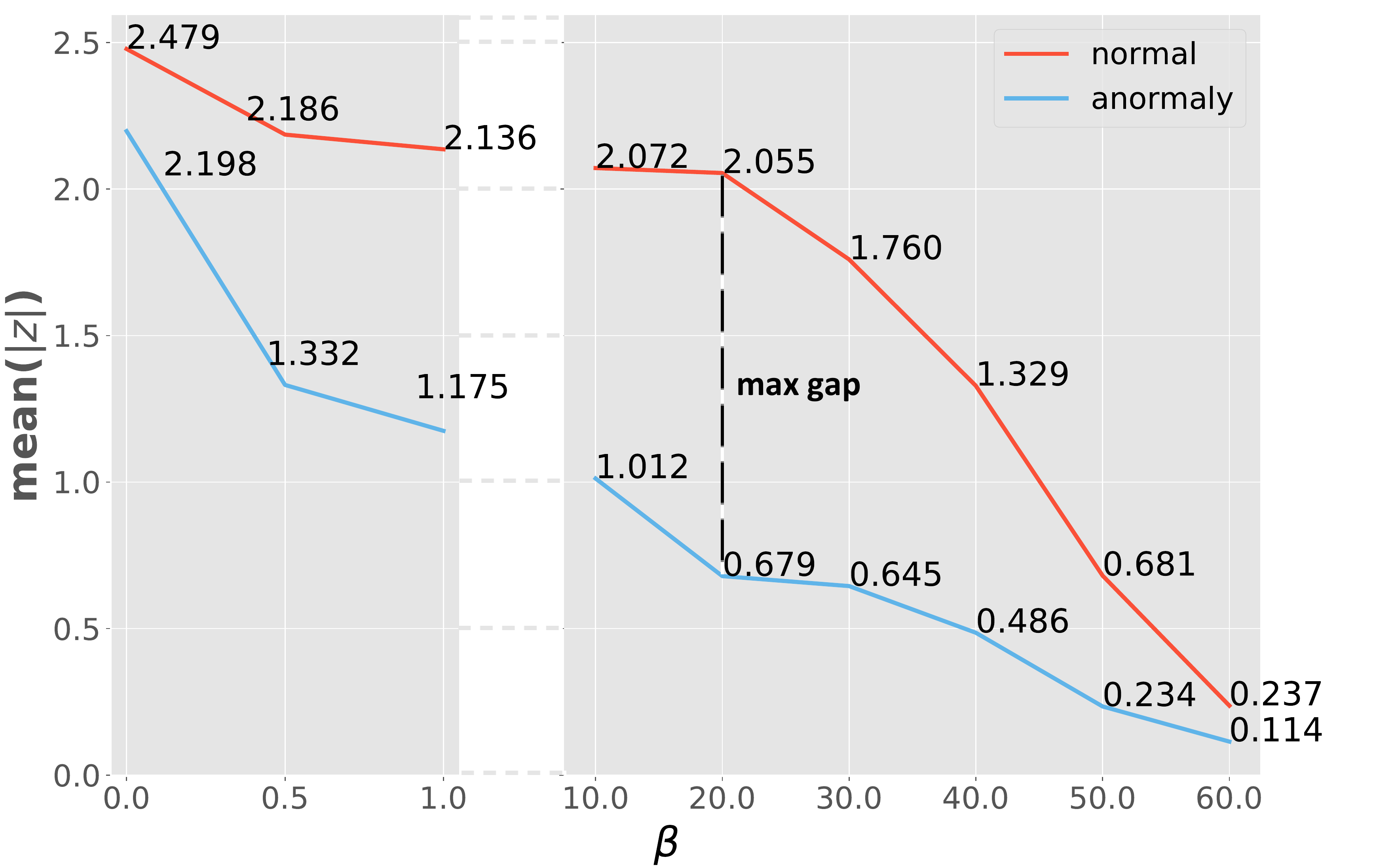}
\caption*{(a) CIFAR10 (normal class: 0)}
\end{minipage}
\hfill
\begin{minipage}[t]{0.48\textwidth}
\centering
\includegraphics[width=\textwidth]{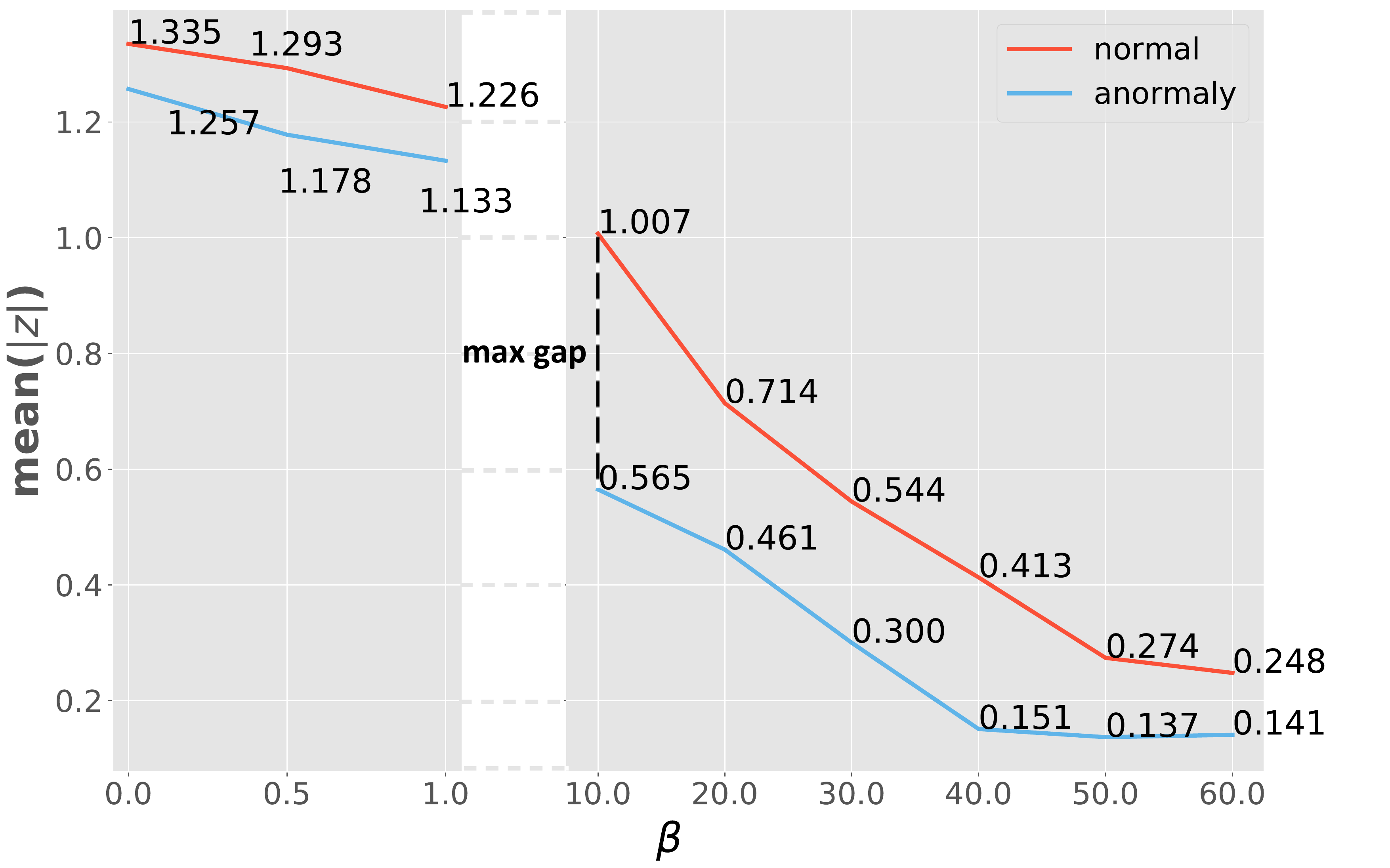}
\caption*{(b) CIFAR100 (normal class: 0)}
\end{minipage}
\caption{The trade-off between MI and entropy. (\textbf{Top}) The AUROC curve (solid line in blue), the loss curve of NCE (dashed dot line in red) and the loss curve of entropy (dashed dot line in purple) \textit{w.r.t.} $\beta$. (\textbf{Bottom}) The mean value of $\|z\|$ \textit{w.r.t.} $\beta$ that corresponds to normal and anomalous data. (\textit{Best viewed in color}) }\label{img:gap}
\end{figure*}

\subsubsection{\textbf{Digging out the essence of the trade-off}}\label{sec:trade-off}

We experiment on CIFAR10 and CIFAR100 with our base model under the hyper-parameter $\beta$ in $\mathcal{B}$. 
For each training class, we investigate the relation between the converged training loss (mutual information and entropy) and the corresponding AUROC. Then we look into the learned representation $z_i$ by calculating their distance from the center $\| z_i \|_2$ of normal and anomalous data as $\beta$ increase. We report the results that trained with class 0 in Fig.~\ref{img:gap}.
In the top panels of Fig.~\ref{img:gap}, we can observe that as $\beta$ increases, the loss corresponds to the mutual information converges to a larger value while the loss corresponds to the entropy converges to a smaller value. As $\beta$ increases, the model tends to ignore the mutual information, which matches our observation that the  MI loss is getting larger and the entropy loss is getting smaller. It is remarkable that $\mathcal{L}_{NCE}$ and $\mathcal{L}_{H}$ become in similar scale when $\beta$ is in the range from $0.5$ to $40.0$ in CIFAR10 and from $10.0$ to $30.0$ in CIFAR100, in which model results in better performance as shown in the top panels in Fig.~\ref{img:gap}. 
This indicates that a proper $\beta$ for the trade-off should lead to similar converged scale for both $\mathcal{L}_{NCE}$ and $\mathcal{L}_{H}$.

In the bottom panels of Fig.~\ref{img:gap}, we compare the mean $\| z_i \|_2$ of all samples between normal and anomalous data in testing datasets.  As $\beta$ increases, the mean $\| z_i \|_2$ for both normal and anomalous data shows a declining tendency. More importantly, the $\| z_i \|_2$ for the anomalous data decreases to a larger extent than that of the normal data. This results in a gap between normal and anomaly data.
We observe that the model results in better performance when the mean $\| z_i \|_2$ gap is larger, with $\beta$ in the range of $0.5$ to $40.0$ in CIFAR10 and $10.0$ to $30.0$ in CIFAR100. Especially, the biggest gap is attained when $\beta = 20.0$ on CIFAR10 and $\beta=10.0$ on CIFAR100 dataset respectively, which correspond to their best performance as shown in the top panels of Fig.~\ref{img:gap}. 

To analyze these experiments, we connect these results with the formulation. As shown in Eq.~\ref{eq:obj_radius}, the entropy regularizer is lower-bounded by the expected value of $\log r(\zv)$. As we choose $r$ as a zero-centered reference distribution, whose density function is proportional to the p-norm $\| z_i \|_p \ (p=1 \ or \ p=2)$, a geometric interpretation is also guaranteed: the model will regularize the Euclidian (\textit{resp.} Manhattan) distance from the center $\| z_i \|_2$ and encourage the representations $z_i$ to be centered. From Fig.~\ref{img:gap}, as increase the value of $\beta$, we can observe that for anomaly data the mean $\| z_i \|_2$ is getting close to zero rapidly, while for normal data which profits from the mutual information maximization the mean $\| z_i \|_2$ is getting close to zero relatively slower. When $\beta$ is too large (\textit{e.g.} $\beta \geqslant 60.0$), we can observe that the maximization of MI is over-regularized, thus the model cannot extract effective features for normal data, which results in a lower AUROC. These results enlighten the importance of the trade-off between mutual information and entropy.

\subsection{Ablation: Mutual Information and Entropy Estimators}
The previous section discusses the importance of the information trade-off. A nice property of our formulation is that Eq.~\ref{eq:LB_KL_reformulate_2_final} is general enough and is able to plug in alternative estimators for the mutual information and entropy loss. In this section, we conduct the ablation experiments with the alternative mutual information and entropy estimator and justify the expressiveness of the estimators that we choose.

\subsubsection{\textbf{Learning with a different mutual information estimator}}\label{sec:whatdoHdo}
We apply the InfoNCE estimator \cite{gutmann2010noise} for the calculation of the mutual information. Considering there are a wide range of approaches for the MI estimation, we also consider an alternative lower bound for MI estimation called Donsker-Varadahn (DV) bound \cite{donsker1975large}. 
\begin{equation}I(\mathbf{x} ; \mathbf{z}) \geqslant \mathbb{E}_{p(\xv, \zv)}[f(\xv, \zv)]-\log \mathbb{E}_{p(\zv)}[e^{f(\xv,\zv)}] \triangleq I_{\mathrm{DV}}
\end{equation}
This bound is widely used in various algorithms such as Mutual Information Neural Estimator (MINE) \cite{MINE}. Deep InfoMax (DIM) \cite{DIM} continues to extend the critic in this bound with a Jensen-Shannon based formulation as:
\begin{equation}I(\mathbf{x} ; \mathbf{z}) \geqslant \mathbb{E}\left[-\sigma\left(-f\left(\xv, \zv\right)\right) -\sigma\left(f\left(\xv^{\prime}, \zv\right)\right)\right] \triangleq I_{\mathrm{JSD}}
\end{equation}
where $\sigma(z) = \log \left(1+e^{z}\right)$ denotes the Softplus function, and note that for DV bound and JSD bound the critic function $f$ is usually modeled with a discriminator function $f: \mathcal{X}\times \mathcal{Z} \rightarrow \mathbb{R}$. Moreover, according to existing studies like \cite{DIM}, the JSD-based estimator and the DV-based estimator behave similarly. Here, we apply the JSD-based MI estimator to represent the effect of DV bound.

To compare the JSD-based and NCE-based MI estimator, on our base model, we utilize the JSD-based MI estimator and the InforNCE estimator. The experiments are conducted on CIFAR10 and CIFAR100 and the hyper-parameter is set as $\beta=20.0$. The results are shown in Table~\ref{tal:Different Optimization Method}. Given the same entropy estimator, we can find the model with the JSD-based MI estimator underperforms than the NCE-based MI estimator. 

To investigate the performance gap, we investigate the training process and illustrate the testing AUROC curve and the mutual information loss during the training process in Fig.~\ref{img:JSD_NCE}, where the NCE-based and the JSD-based estimator is marked in red and blue respectively. Compared to the NCE-based estimator, both testing AUROC and the loss by JSD-based estimator show large perturbation during the training, which makes the model more difficult to converge. 

One plausible explanation is that the mutual information estimation variance of the JSD estimator is much larger than the NCE estimator, which is the critical difference between these two MI estimators \cite{poole2019variational}. This large variance caused by the JSD-based estimator makes the model more unstable during the training and the learned representations may be useless. In this way, the JSD-based methods need more fine-tuning to stabilize the training process of the model.

\begin{table}[!t]
    \centering
    \fontsize{10}{12}\selectfont    
    \caption{Comparison of different MI and entropy estimators.}
    \label{tal:Different Optimization Method}
    \begin{tabular}{ccccc}
    \toprule
    \toprule
    \multirow{2}{*}{Loss function}&
    \multicolumn{2}{c}{CIFAR10}&\multicolumn{2}{c}{CIFAR100}\cr
    \cmidrule(lr){2-3}\cmidrule(lr){4-5}
    &AVG&STD&AVG&STD\cr
    \cmidrule(lr){1-5}
     $\mathcal{L}_{NCE} + \beta \cdot \mathcal{L}_{2\;Entropy}$ & 85.5 & 7.02 & 76.9 & 5.26\cr
     $\mathcal{L}_{NCE} + \beta \cdot \mathcal{L}_{1\;Entropy}$ & 86.7 & 5.47 & 77.4 & 4.62\cr
     $\mathcal{L}_{JSD} + \beta \cdot \mathcal{L}_{2\;Entropy}$ & 72.5 & 7.20 & 69.6 & 4.73\cr
     $\mathcal{L}_{JSD} + \beta \cdot \mathcal{L}_{1\;Entropy}$ & 73.9 & 5.00 & 72.7 & 3.21\cr
    \bottomrule
    \bottomrule
    \end{tabular}\vspace{0cm}
\end{table}

\begin{table}[!t]
    \centering
    \fontsize{10}{12}\selectfont    
    \caption{Efficient normal scoring mechanism.}
    \label{tal:Efficient normal scoring mechanism}
    \begin{tabular}{ccccc}
    \toprule
    \toprule
    \multirow{2}{*}{Normal scoring function}&
    \multicolumn{2}{c}{CIFAR10}&\multicolumn{2}{c}{CIFAR100}\cr
    \cmidrule(lr){2-3}\cmidrule(lr){4-5}
    &AVG&STD&AVG&STD\cr
    \cmidrule(lr){1-5}
     $NormalScore_{rand}$ & 86.4 & 1.04 & 75.6 & 1.26\cr
     $NormalScore_{mc}$ & 86.9 & 0.05 & 78.5 & 0.12\cr
     $NormalScore_{ori}$ & 86.7 & $-$ & 77.4 & $-$\cr
    \bottomrule
    \bottomrule
    \end{tabular}\vspace{0cm}
\end{table}

\begin{figure}[t]
\centering
\includegraphics[width=8cm]{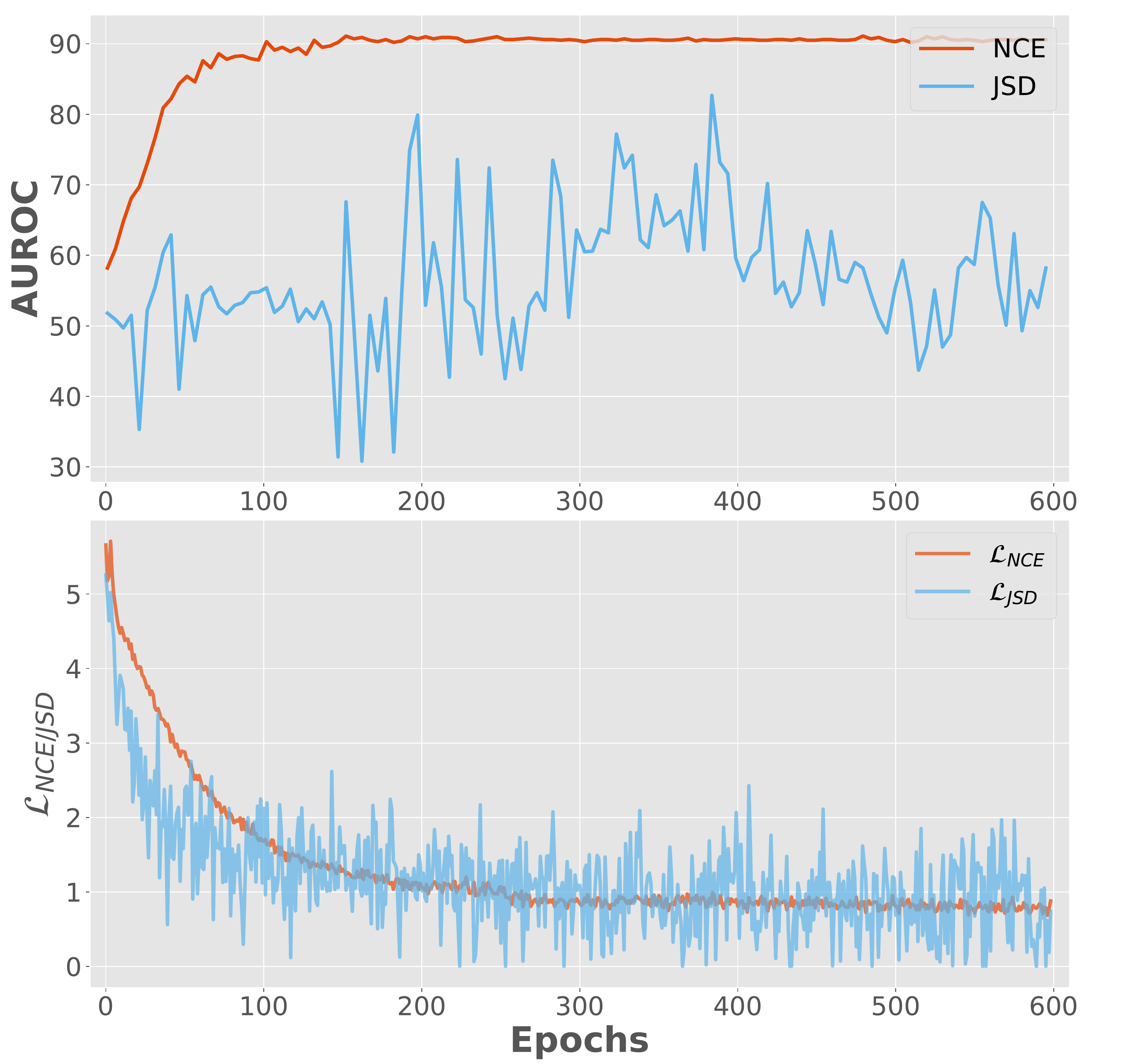}
\caption{The comparison of NCE-based and JSD-based mutual information estimator: conducted on dataset CIFAR10 where class 1 (car) is considered as normal data. (\textbf{Top}) The testing ROC curve during the training process by using NCE and JSD estimator for mutual information calculation. (\textbf{Bottom}) The estimated loss by using the NCE-based and JSD-based mutual information estimator during the training process.}
\label{img:JSD_NCE}
\end{figure}

\subsubsection{\textbf{Learning with a different entropy estimator}}\label{sec:whichloss}

We introduce two reference distributions in Eq.~\ref{eq:obj_activate} and Eq.~\ref{eq:obj_activate_1} to estimate the entropy, corresponding to the L2 norm and the L1 norm, respsectively. Similar to the previous experiments, CIFAR10 and CIFAR100 are used. With different entropy estimators, we test with both the NCE-based and the JSD-based mutual information estimator, and report the mean and standard deviation of the AUROC among all classes in Table~\ref{tal:Different Optimization Method}. Given the same MI estimator, the L1 norm estimator yields better performance. A case-by-case study is conducted and based on the hyper-parameter that we set, when the L1 norm used as the entropy regularizer, the representations are more inclined to shrink to the center. This results in a larger distance gap between normal and anomaly data than the L2 norm does and in order to have the same performance L2 norm may need more proper hyper-parameter tuning to balance the information trade-off.

Based on our observation, the NCE-based estimator is more suitable and compatible with the proposed entropy estimator. In the following sections, we continue to refine the NCE-based estimator and apply the L1 entropy regularizer to present the effectiveness of our proposed method.

\renewcommand \arraystretch{0.9}
\begin{table*}[!htb]
\centering
\caption{Average area under the ROC curve (AUROC) in \% of anomaly detection methods. For every dataset, each model is trained on the single class, and tested against all other classes. ``SD'' means standard deviation among classes. The best performing method is in bold. Results of Deep SVDD are borrowed from~\cite{SVDD}}. Results of VAE, AnoGAN and ADGAN are from~\cite{deecke2018anomaly}. Results of DAGMM, DSEBM and GeoTrans are from~\cite{golan2018deep}. Results of GANomaly and ARNet are from~\cite{fye2020ARNet}. 
	\small
	\begin{minipage}[t]{0.95\textwidth}
	\begin{tabular}{cx{3.0cm}x{0.6cm}x{0.6cm}x{0.6cm}x{0.6cm}x{0.6cm}x{0.6cm}x{0.6cm}x{0.6cm}x{0.6cm}x{0.6cm}x{0.6cm}x{0.8cm}}
	\toprule
	Dataset & Method & 0 & 1 & 2 & 3 & 4 & 5 & 6 & 7 & 8 & 9 & avg & std\\
	\cmidrule(lr){1-1} \cmidrule(lr){2-2} \cmidrule(lr){3-3} \cmidrule(lr){4-4} \cmidrule(lr){5-5} \cmidrule(lr){6-6} \cmidrule(lr){7-7} \cmidrule(lr){8-8} \cmidrule(lr){9-9} \cmidrule(lr){10-10} \cmidrule(lr){11-11} \cmidrule(lr){12-12} \cmidrule(lr){13-13} \cmidrule(lr){14-14}
		& VAE~\cite{kingma2013auto} 
		& 92.1 & \textbf{99.9} & 81.5 & 81.4 & 87.9 & 81.1 & 94.3 & 88.6 & 78.0 & 92.0 & 87.7 & 7.05\\
		& D-SVDD~\cite{SVDD}
		& 98.0 & 99.7 & 91.7 & 91.9 & 94.9 & 88.5 & 98.3 & 94.6 & 93.9 & 96.5 & 94.8 & 3.46\\
		& AnoGAN~\cite{schlegl2017unsupervised}
		& 99.0 & 99.8 & 88.8 & 91.3 & 94.4 & 91.2 & 92.5 & 96.4 & 88.3 & 95.8 & 93.7 & 4.00\\
		& ADGAN~\cite{deecke2018anomaly} 
		& 99.5 & \textbf{99.9} & 93.6 & 92.1 & 94.9 & 93.6 & 96.7 & 96.8 & 85.4 & 95.7 & 94.7 & 4.15\\
		MNIST& GANomaly~\cite{Akcay2018} 
		& 97.2 & 99.6 & 85.1 & 90.6 & 94.9 & 94.9 & 97.1 & 93.9 & 79.7 & 95.4 & 92.8 & 6.12\\
		& OCGAN~\cite{OCGAN} 
		& \textbf{99.8} & \textbf{99.9} & 94.2 & 96.3 & 97.5 & 98.0 & 99.1 & 98.1 & 93.9 & 98.1 & 97.5 & 2.10\\
		& GeoTrans~\cite{golan2018deep} 
		& 98.2 & 91.6 & \textbf{99.4} & 99.0 & \textbf{99.1} & \textbf{99.6} & \textbf{99.9} & 96.3 & 97.2 & \textbf{99.2} & 98.0 & 2.50\\
		& ARNet~\cite{fye2020ARNet} & 98.6 & \textbf{99.9} & 99.0 & \textbf{99.1} & 98.1 & 98.1 & 99.7 & \textbf{99.0} & 93.6 & 97.8 & \textbf{98.3} & 1.78\\
		\cmidrule(lr){2-2} \cmidrule(lr){3-3} \cmidrule(lr){4-4} \cmidrule(lr){5-5} \cmidrule(lr){6-6} \cmidrule(lr){7-7} \cmidrule(lr){8-8} \cmidrule(lr){9-9} \cmidrule(lr){10-10} \cmidrule(lr){11-11} \cmidrule(lr){12-12} \cmidrule(lr){13-13} \cmidrule(lr){14-14}
        
        &Our Base model & 91.3 & 98.6 & 90.6 & 88.6 & 88.2 & 89.1 & 91.2 & 85.6 & 87.7 & 86.0 & 89.7 & 3.70\\
		& Our Extension model & 99.5 & 99.7 & 98.8 & 98.3 & 97.7 & 96.7 & 98.7 & 97.5 & \textbf{98.6} & 97.3 & \textbf{98.3} & \textbf{0.97}\\
		        \cmidrule(lr){1-14}
		& DAGMM~\cite{zhai2016deep} 
        & 42.1 & 55.1 & 50.4 & 57.0 & 26.9 & 70.5 & 48.3 & 83.5 & 49.9 & 34.0 & 51.8 & 16.47\\
		& DSEBM~\cite{zong2018deep} 
		& 91.6 & 71.8 & 88.3 & 87.3 & 85.2 & 87.1 & 73.4 & 98.1 & 86.0 & 97.1 & 86.6 & 8.61\\
		& ADGAN~\cite{deecke2018anomaly} 
		& 89.9 & 81.9 & 87.6 & 91.2 & 86.5 & 89.6 & 74.3 & 97.2 & 89.0 & 97.1 & 88.4 & 6.75\\
		Fashion- & GANomaly~\cite{Akcay2018} 
		& 80.3 & 83.0 & 75.9 & 87.2 & 71.4 & 92.7 & 81.0 & 88.3 & 69.3 & 80.3 & 80.9 & 7.37\\
		MNIST & GeoTrans~\cite{golan2018deep} 
		& \textbf{99.4} & 97.6 & 91.1 & 89.9 & 92.1 & 93.4 & 83.3 & 98.9 & 90.8 & 99.2 & 93.5 & 5.22\\
		& ARNet~\cite{fye2020ARNet} & 92.7 & 99.3 & 89.1 & 93.6 & 90.8 & 93.1 & 85.0 & 98.4 & 97.8 & 98.4 & 93.9 & 4.70\\
		\cmidrule(lr){2-2} \cmidrule(lr){3-3} \cmidrule(lr){4-4} \cmidrule(lr){5-5} \cmidrule(lr){6-6} \cmidrule(lr){7-7} \cmidrule(lr){8-8} \cmidrule(lr){9-9} \cmidrule(lr){10-10} \cmidrule(lr){11-11} \cmidrule(lr){12-12} \cmidrule(lr){13-13} \cmidrule(lr){14-14}
        
        & Our Base model & 94.3 & 99.1 & 91.6 & 95.3 & 90.9 & 98.0 & 86.1 & 98.0 & 97.2 & 97.1 & 94.8 & 4.11\\
		& Our Extension model & 96.7 & \textbf{99.7} & \textbf{95.3} & \textbf{97.3} & \textbf{95.1} & \textbf{99.2} & \textbf{89.8} & \textbf{99.3} & \textbf{99.1} & \textbf{99.3} & \textbf{97.1} & \textbf{3.08}\\
        \cmidrule(lr){1-14}
	& VAE~\cite{kingma2013auto} 
	& 62.0 & 66.4 & 38.2 & 58.6 & 38.6 & 58.6 & 56.5 & 62.2 & 66.3 & 73.7 & 58.1 & 11.50\\
	& D-SVDD~\cite{SVDD}
	& 61.7 & 65.9 & 50.8 & 59.1 & 60.9 & 65.7 & 67.7 & 67.3 & 75.9 & 73.1 & 64.8 & 7.16\\
	& DAGMM~\cite{zhai2016deep} 
	& 41.4 & 57.1 & 53.8 & 51.2 & 52.2 & 49.3 & 64.9 & 55.3 & 51.9 & 54.2 & 53.1 & 5.95\\
	& DSEBM~\cite{zong2018deep} 
	& 56.0 & 48.3 & 61.9 & 50.1 & 73.3 & 60.5 & 68.4 & 53.3 & 73.9 & 63.6 & 60.9 & 9.10\\
	& AnoGAN~\cite{schlegl2017unsupervised} 
	& 61.0 & 56.5 & 64.8 & 52.8 & 67.0 & 59.2 & 62.5 & 57.6 & 72.3 & 58.2 & 61.2 & 5.68\\
	CIFAR- & ADGAN~\cite{deecke2018anomaly} 
	& 63.2 & 52.9 & 58.0 & 60.6 & 60.7 & 65.9 & 61.1 & 63.0 & 74.4 & 64.4 & 62.4 & 5.56\\
	10 & GANomaly~\cite{Akcay2018} 
	& \textbf{93.5} & 60.8 & 59.1 & 58.2 & 72.4 & 62.2 & 88.6 & 56.0 & 76.0 & 68.1 & 69.5 & 13.08\\
	& OCGAN~\cite{OCGAN} 
	& 75.7 & 53.1 & 64.0 & 62.0 & 72.3 & 62.0 & 72.3 & 57.5 & 82.0 & 55.4 & 65.6 & 9.52\\
	& GeoTrans~\cite{golan2018deep} 
	& 74.7 & 95.7 & 78.1 & 72.4 & 87.8 & 87.8 & 83.4 & \textbf{95.5} & 93.3 & 91.3 & 86.0 & 8.52\\
	& ARNet~\cite{fye2020ARNet} & 78.5 & 89.8 & \textbf{86.1} & 77.4 & \textbf{90.5} & 84.5 & 89.2 & 92.9 & 92.0 & 85.5 & 86.6 & 5.35\\
	\cmidrule(lr){2-2} \cmidrule(lr){3-3} \cmidrule(lr){4-4} \cmidrule(lr){5-5} \cmidrule(lr){6-6} \cmidrule(lr){7-7} \cmidrule(lr){8-8} \cmidrule(lr){9-9} \cmidrule(lr){10-10} \cmidrule(lr){11-11} \cmidrule(lr){12-12} \cmidrule(lr){13-13} \cmidrule(lr){14-14}
    
    & Our Base model & 89.9 & 95.7 & 80.6 & 72.8 & 80.9 & 83.8 & 91.1 & 91.1 & 93.0 & 87.9 & 86.7 & 5.47\\
	& Our Extension model & 93.0 & \textbf{97.6} & 88.5 & \textbf{85.8} & 89.5 & \textbf{92.1} & \textbf{95.2} & 93.6 & \textbf{95.0} & \textbf{95.5} & \textbf{92.6} & \textbf{3.65}\\
	\cmidrule(lr){1-14}
	    & GANomaly~\cite{Akcay2018} 
		& 58.9 & 57.5 & 55.7 & 57.9 & 47.9 & 61.2 & 56.8 & 58.2 & 49.7 & 48.8 & 55.3 & \textbf{4.46}\\
		& GeoTrans~\cite{golan2018deep} 
		& 72.9 & 61.0 & 66.8 & 82.0 & 56.7 & 70.1 & 68.5 & 77.2 & 62.8 & 83.6 & 70.1 & 8.43\\
		ImageNet & ARNet~\cite{fye2020ARNet} & 71.9 & 85.8 & 70.7 & 78.8 & 69.5 & 83.3 & 80.6 & 72.4 & 74.9 & \textbf{84.3} & 77.2 & 5.77\\
		\cmidrule(lr){2-2} \cmidrule(lr){3-3} \cmidrule(lr){4-4} \cmidrule(lr){5-5} \cmidrule(lr){6-6} \cmidrule(lr){7-7} \cmidrule(lr){8-8} \cmidrule(lr){9-9} \cmidrule(lr){10-10} \cmidrule(lr){11-11} \cmidrule(lr){12-12} \cmidrule(lr){13-13} \cmidrule(lr){14-14}
        
        & Our Base model & 79.8 & 81.3 & 80.7 & 85.4 & 79.7 & 84.1 & 75.3 & 81.5 & 76.1 & 67.1 & 79.1 & 5.23\\
        
		& Our Extension model & \textbf{88.2} & \textbf{89.8} & \textbf{86.1} & \textbf{89.6} & \textbf{89.4} & \textbf{89.4} & \textbf{80.3} & \textbf{85.7} & \textbf{81.8} & 74.9 & \textbf{85.5} & 5.02\\
		\cmidrule(lr){1-14}
	\end{tabular}
	\end{minipage}
	
	\begin{minipage}[t]{0.92\textwidth}
	\small
	\begin{tabular}{cx{3.0cm}x{0.7cm}x{0.7cm}x{0.7cm}x{0.7cm}x{0.7cm}x{0.7cm}x{0.7cm}x{0.7cm}x{0.7cm}x{0.7cm}x{0.8cm}}
		Dataset & Method & 0 & 1 & 2 & 3 & 4 & 5 & 6 & 7 & 8 & 9 & 10\\
		\cmidrule(lr){1-1} \cmidrule(lr){2-2} \cmidrule(lr){3-3} \cmidrule(lr){4-4} \cmidrule(lr){5-5} \cmidrule(lr){6-6} \cmidrule(lr){7-7} \cmidrule(lr){8-8} \cmidrule(lr){9-9} \cmidrule(lr){10-10} \cmidrule(lr){11-11} \cmidrule(lr){12-12} \cmidrule(lr){13-13}
		& DAGMM~\cite{zhai2016deep} 
		& 43.4 & 49.5 & 66.1 & 52.6 & 56.9 & 52.4 & 55.0 & 52.8 & 53.2 & 42.5 & 52.7\\
		& DSEBM~\cite{zong2018deep} 
		& 64.0 & 47.9 & 53.7 & 48.4 & 59.7 & 46.6 & 51.7 & 54.8 & 66.7 & 71.2 & 78.3 \\
		& ADGAN~\cite{deecke2018anomaly} 
		& 63.1 & 54.9 & 41.3 & 50.0 & 40.6 & 42.8 & 51.1 & 55.4 & 59.2 & 62.7 & 79.8 \\
		& GANomaly~\cite{Akcay2018} 
		& 57.9 & 51.9 & 36.0 & 46.5 & 46.6 & 42.9 & 53.7 & 59.4 & 63.7 & 68.0 & 75.6\\
		& GeoTrans~\cite{golan2018deep} 
		& 74.7 & 68.5 & 74.0 & 81.0 & 78.4 & 59.1 & 81.8 & 65.0 & \textbf{85.5} & 90.6 & 87.6\\
		& ARNet~\cite{fye2020ARNet} & 77.5 & 70.0 & 62.4 & 76.2 & 77.7 & 64.0 & \textbf{86.9} & 65.6 & 82.7 & 90.2 & 85.9 \\
		\cmidrule(lr){2-2} \cmidrule(lr){3-3} \cmidrule(lr){4-4} \cmidrule(lr){5-5} \cmidrule(lr){6-6} \cmidrule(lr){7-7} \cmidrule(lr){8-8} \cmidrule(lr){9-9} \cmidrule(lr){10-10} \cmidrule(lr){11-11} \cmidrule(lr){12-12} \cmidrule(lr){13-13}
		&Our Base model & 71.1 & 80.2 & 78.8 & 79.5 & 78.9 & 79.6 & 78.5 & 75.6 & 74.3 & 84.8 & 85.2\\
		CIFAR-& Our Extension model & \textbf{82.0} & \textbf{85.8} & \textbf{87.8} & \textbf{86.2} & \textbf{89.3} & \textbf{87.6} & 86.8 & \textbf{83.0} & 85.1 & \textbf{91.3} & \textbf{90.1} \\
		\cmidrule(lr){2-13}
		100 & Method & 11 & 12 & 13 & 14 & 15 & 16 & 17 & 18 & 19 & \textbf{avg} & SD\\
		\cmidrule(lr){2-2} \cmidrule(lr){3-3} \cmidrule(lr){4-4} \cmidrule(lr){5-5} \cmidrule(lr){6-6} \cmidrule(lr){7-7} \cmidrule(lr){8-8} \cmidrule(lr){9-9} \cmidrule(lr){10-10} \cmidrule(lr){11-11} \cmidrule(lr){12-12} \cmidrule(lr){13-13} 
		& DAGMM~\cite{zhai2016deep}  
		& 46.4 & 42.7 & 45.4 & 57.2 & 48.8 & 54.4 & 36.4 & 52.4 & 50.3 & 50.5 & 6.55\\
		& DSEBM~\cite{zong2018deep}  
		& 62.7 & 66.8 & 52.6 & 44.0 & 56.8 & 63.1 & 73.0 & 57.7 & 55.5 & 58.8 & 9.36\\
		& ADGAN~\cite{deecke2018anomaly} 
		& 53.7 & 58.9 & 57.4 & 39.4 & 55.6 & 63.3 & 66.7 & 44.3 & 53.0 & 54.7 & 10.08\\
		& GANomaly~\cite{Akcay2018} 
		& 57.6 & 58.7 & 59.9 & 43.9 & 59.9 & 64.4 & 71.8 & 54.9 & 56.8 & 56.5 & 9.94\\
		& GeoTrans~\cite{golan2018deep} 
		& \textbf{83.9} & 83.2 & 58.0 & \textbf{92.1} & 68.3 & 73.5 & 93.8 & 90.7 & 85.0 & 78.7 & 10.76 \\
		& ARNet~\cite{fye2020ARNet} & 83.5 & \textbf{84.6} & 67.6 & 84.2 & 74.1 & 80.3 & 91.0 & 85.3 & 85.4 & 78.8 & 8.82\\
		\cmidrule(lr){2-2} \cmidrule(lr){3-3} \cmidrule(lr){4-4} \cmidrule(lr){5-5} \cmidrule(lr){6-6} \cmidrule(lr){7-7} \cmidrule(lr){8-8} \cmidrule(lr){9-9} \cmidrule(lr){10-10} \cmidrule(lr){11-11} \cmidrule(lr){12-12} \cmidrule(lr){13-13}
        &Our Base model & 73.4 & 74.2 & 72.1 & 72.8 & 73.1 & 69.8 & 84.0 & 80.3 & 81.0 & 77.4 & 4.62 \\
        & Our Extension model & 80.9 & 84.2 & \textbf{80.3} & 85.2 & \textbf{79.7} & \textbf{82.2} & \textbf{93.9} & \textbf{91.5} & \textbf{87.9} & \textbf{86.0} & \textbf{3.98} \\
		\bottomrule
	\end{tabular}
	\end{minipage}
\label{tal:AUC1}
\end{table*}

\subsection{Efficient Normal Scoring Mechanism}\label{sec:EfficientScoring}

As shown in Eq.~\ref{eq:anomaly score} the normal score is firstly designed to be the similarity of $z_i$ and $\widetilde{z}_i$, which are the latent representations of the randomly augmented view $x_i$ and $\widetilde{x}_i$ of the same test data example $x_{ori_i}$. However, the random augmentation yield in unstable result. An alternative solution is to utilized Monte Carlo samples, and the normal score is reformulated as:
\begin{align}
\begin{split}\label{eq:anomaly score MonteCarlo}
{\operatorname{Normal Score}}=\sum^{H}_{h=1}[sim(z_{g_i}^{(h)},\widetilde{z}_{g_i}^{(h)})]
\end{split}
\end{align}
Where the $h$ refers to the $h^{st}$ sampling. Despite being more stable, it is time-consuming. We introduce a compromise as shown in Eq.~\ref{eq:anomaly score final}, which skip the augmentation and directly input the original sample.

With the base model, we conduct experiment on CIFAR10 and CIFAR100 to evaluate the three normal scoring mechanism based on Eq.~\ref{eq:anomaly score}, \ref{eq:anomaly score final} and \ref{eq:anomaly score MonteCarlo}, which is noted as  $NormalScore_{rand}$, $NormalScore_{ori}$ and $NormalScore_{mc}$ accordingly. To be noted that, in Eq.~\ref{eq:anomaly score MonteCarlo}, H is set as 100. We repeat the testing process for 10 times for each normal scoring mechanism and record the average and standard deviation of the AUROC of all classes in CIFAR10. As illustrated in Table~\ref{tal:Efficient normal scoring mechanism}, the normal scoring mechanisms base on Eq.~\ref{eq:anomaly score final} yield a comparable and stable result with only one inference. We take Eq.~\ref{eq:anomaly score final}$/$Eq.~\ref{eq:anomaly score final extention} as our refined normal scoring mechanism for the base$/$extension model through in the following experiments to obtain the best performance.

\subsection{Extensive Experiment on Extension Model}\label{sec:Extensive Experiment}

Table~\ref{tal:AUC1} provides results of our base model and extension model on MNIST, Fashion-MNIST, CIFAR-10, ImageNet and CIFAR-100. 
For grayscale datasets, such as MNIST and Fashion-MNIST, we convert the image into 3 channels through expanding.
The performance of extension model is improved considerably over the base model both in average and standard deviation of the AUROC with Local DIM~\cite{AMDIM}, indicating that a more effective module to optimize mutual information is beneficial.
On all involved datasets, experiment results illustrate that the average AUROC or the standard deviation of our method outperforms all other methods to different extents. Furthermore, our method reveals greater advantage in more difficult datasets, from FashionMNIST which surpasses the SOTA by 3.2\%, to 6\% in CIFAR10, to 7.2\% in CIFAR100 and to 8.3\% in ImageNet Subset. This may be because our method is benefit from a decoder-free representation learning framework, where an advanced feature extracting model like RESNET can be utilized to handle pictures with higher resolution and more complex texture.
More importantly, our method results in the lower standard deviation of AUROCs, which reflect the stability of the model when dealing with different kinds of anomalous data. This is vital especially in anomaly detection, where anomalous data can not be foreseen\cite{fye2020ARNet}.

\section{Conclusion and Future Work}
In this paper, we are the first to put forward an anomaly detection based objective function that can be optimized end-to-end in an unsupervised fashion. Many works can find mathematical support and further discover the potential optimization direction through our method. At last, based on the objective function we present a method that overperforms the state-of-the-arts, which illustrates the correctness of our objective function and rationality of the design of the loss function.
Despite we restrict our method in an unsupervised anomaly detection setting, as can been seen in Equation \ref{eq:LB_KL_reformulate}, this work can be extended to semi-supervised anomaly detection, which remains to be our future work. Notably, there are other loss functions in maximizing mutual information and minimizing entropy to explore.
Furthermore, we can investigate the specific functions of each component in our lower bound objective function in Eq.~\ref{eq:LB_KL_reformulate_2_final}. The function of mutual information seems obvious, as widely explored in unsupervised representation learning approaches~\cite{CPC,DIM,AMDIM,SimCLR,MoCo1}, maximizing mutual information can effectively force the model to obtain better representation. Since the entropy forces $\|z_i\|_2$ to be closed to zero, it may be indicated that the entropy regularizer limits the model representation ability. Thus this trade-off seems to force the network to represent normal data with limited representation ability. In another word, this trade-off force the network to generate features only to properly represent normal data, which is consistent with the insight in \cite{OCGAN}. As anomalous data is inaccessible during training, one feasible solution to enlarge the distribution of normal and anomalous data in latent space is to make anomalous data unable to be represented properly. Is this the essence of anomaly detection? We will explore this in our future work.
.
\ifCLASSOPTIONcaptionsoff
  \newpage
\fi

\bibliographystyle{IEEEtran}
\bibliography{reference.bib}

\clearpage

\appendices

\section{Model Structure}
The detailed structure of the model we used can be found in this section, where the model structure in Table~\ref{tal:big structure} is utilized for dataset ImageNet and model structure in Table~\ref{tal:small structure} is utilized for other datasets.

\begin{table}[!htb]
 \centering
 \caption{Small Encoder Architecture}\label{tal:small structure}
 \small
\begin{tabular}{l}
\toprule \textbf{Small Encoder Architecture} \\
\hline $\operatorname{ReLU}(\text { Conv } 2 \mathrm{d}(3, \mathrm{ndf}, 3,1,0))$ \\
ResBlock $\left(\mathrm{ndf}, \mathrm{ndf}, 1,1,0\right)$\\
ResBlock $\left(1^{*} \mathrm{ndf}, 2^{*} \mathrm{ndf}, 4,2, \text { ndepth }\right)$ \\
ResBlock $\left(2^{*} \mathrm{ndf}, 4^{*} \mathrm{ndf}, 2,2, \text { ndepth }\right)$ \\
ResBlock $\left(4^{*} \mathrm{ndf}, 4^{*} \mathrm{ndf}, 3,1, \text { ndepth }\right)-$ provides $f_{local}$ \\
ResBlock $\left(4^{*} \mathrm{ndf}, 4^{*} \mathrm{ndf}, 3,1, \text { ndepth }\right)$ \\
ResBlock $\left(4^{*} \mathrm{ndf}, \text { nrkhs }, 3,1,1\right)$ \ \ \ \ \ \ \ $-$ provides $f_{global}$ \\
\hline $ndf=128,~nrkhs=1024,~ndepth=10$\\
\bottomrule
\end{tabular}
\end{table}

\begin{table}[!htb]
 \centering
 \caption{Big Encoder Architecture}\label{tal:big structure}
 \small
\begin{tabular}{l}
\toprule \textbf{Big Encoder Architecture} \\
\hline $\operatorname{ReLU}(\text { Conv } 2 \mathrm{d}(3, \mathrm{ndf}, 5,2,2))$ \\
$\operatorname{ReLU}(\text { Conv } 2 \mathrm{d}(\mathrm{ndf}, \mathrm{ndf}, 3,1,0))$ \\
ResBlock $\left(1^{*} \mathrm{ndf}, 2^{*} \mathrm{ndf}, 4,2, \text { ndepth }\right)$ \\
ResBlock $\left(2^{*} \mathrm{ndf}, 4^{*} \mathrm{ndf}, 4,2, \text { ndepth }\right)$ \\
ResBlock $\left(4^{*} \mathrm{ndf}, 8^{*} \mathrm{ndf}, 2,2, \text { ndepth }\right)$ \\
ResBlock $\left(8^{*} \mathrm{ndf}, 8^{*} \mathrm{ndf}, 3,1, \text { ndepth }\right)-$ provides $f_{local}$ \\
ResBlock $\left(8^{*} \mathrm{ndf}, 8^{*} \mathrm{ndf}, 3,1, \text { ndepth }\right)$ \\
ResBlock $\left(8^{*} \mathrm{ndf}, \text { nrkhs }, 3,1,1\right)$ \ \ \ \ \ \ \ $-$ provides $f_{global}$ \\
\hline $ndf=192,~nrkhs=1536,~ndepth=8$\\
\bottomrule
\end{tabular}
\end{table}

\section{Pseudocode for extension method}
\label{section:Pseudocode for extension method}

The testing pseudocode of base model can be found in Algorithm~\ref{alg::Testing Pseudocode basic model}. The training and testing pseudocode of extension model can be found in Algorithm~\ref{alg::Training Pseudocode extension model},~\ref{alg::Testing Pseudocode extension model}. The NCE Loss pseudocode of extension model can be found in Algorithm~\ref{alg::NCE Loss Pseudocode extension model}.

\begin{algorithm}
  \caption{Testing Pseudocode for base model}
  \label{alg::Testing Pseudocode basic model}  
  \footnotesize
  \KwIn{batch size N, similarity function $sim$, structure of model $E_\theta$}
    \For{sampled minibatch $\left\{x_k\right\}_{k=1}^N$}
    {
        \For{\textbf{all} $k \in \left\{ 1,....,N\right\}$}
        {
            \# Extract features\\
            $z_{ki} = E_\theta(x_i)$\\
        }
        \For{ \textbf{all} $i \in \left\{ 1,....,N\right\}$ }
        {   
            \# Similarity\\
            $s_i = sim(z_i,z_i)=z^\top_i \cdot z_i$\\
            $s'_{i} = c_2\cdot tanh\left(\frac{s_{i}}{c_1\cdot c_2} \right)$ \\
            \#normal data has larger $s'_i$ \\
        }
    }
    return $roc\_auc\_score(s'_i,\ label_i)$
\end{algorithm}

\begin{algorithm}
  \caption{NCE Loss Pseudocode extension model}
  \label{alg::NCE Loss Pseudocode extension model} 
  \footnotesize
  \KwIn{batch size N, $input\ feature\;\Phi_1,\Phi_2$}

  \# $shape \ \Phi_1 (n\_batch, n\_dim, n_1)$\\
  \# $shape \ \Phi_2 (n\_batch, n\_dim, n_2)$\\
  \For{sampled minibatch $\left\{x_k\right\}_{k=1}^N$}
  {
        \For{ \textbf{all} $i \in \left\{ 1,....,2N\right\} and\; j \in \left\{ 1,....,2N\right\} $}
        {   
            \# $shape \ \Phi_1 (n\_dim, n_1)$\\
            \# $shape \ \Phi_2 (n\_dim, n_2)$\\
            \# Similarity\\
            $s_{i,j} = sim(\phi_{1_i},\phi_{2_j}) = \sum_{n_1}\sum_{n_2} \phi_{1_i}^\top \cdot \phi_{2_j}$\\
            $s'_{i,j} = c_2\cdot tanh\left(\frac{s_{i,j}}{c_1\cdot c_2} \right)$ \\
        }
        
        \textbf{define} $S = \left\{ s'_{i,j} ,i,j \in {2N} \right\},\; s_{max} = \max\left\{S\right\}$\\
        \textbf{define} $S_{shift} = \left\{ \hat{s}_{i,j} = s'_{i,j}-s_{max} ,i,j \in {2N} \right\}$
        
        \textbf{define} $\ell(i,j) =-\log \frac{exp\left(\hat{s}_{i,j}\right)}{ \sum_{k=1}^{2N}\mathbb{I}_{[k\not=i]} exp(\hat{s}_{i,k})} $\\
        $\mathcal{L}_{nce}=\frac{1}{2N}\sum_{k=1}^N\left[ \ell(2k-1,2k)+\ell(2k,2k-1)\right]$\\
    }
    $return \mathcal{L}_{nce}$ 
\end{algorithm}

\begin{algorithm}
  \caption{Training Pseudocode extension model}
  \label{alg::Training Pseudocode extension model} 
  \footnotesize
  \KwIn{batch size N, NCE Loss function $\mathcal{L}_{nce}$, encoder model $E_{\theta_1}$, nonlinear function $ \phi_{\theta_2}$ entropy weight $\beta$}
  \For{sampled minibatch $\left\{x_k\right\}_{k=1}^N$}
  {
        \For{\textbf{all}$ k \in \left\{ 1,....,N\right\}$}
        {
            randomly draw two augmentation functions $t, t'\sim \mathcal{T}$\\
            \# Augmentation\\
            $\widetilde{x}_{2k-1} = t(x_k)$\\
            $\widetilde{x}_{2k} = t'(x_k)$\\
            \# Extract features\\
            $g_{2k-1},l_{2k-1} = E_{\theta_1}(\widetilde{x}_{2k-1})$\\
            $g_{2k},l_{2k} = E_{\theta_1}(\widetilde{x}_{2k})$\\
            \# nonlinear projection\\
            $\widetilde{l}_{2k-1} = \phi_{\theta_2}(l_{2k-1})$\\
            $\widetilde{l}_{2k} = \phi_{\theta_2}(l_{2k})$\\
            
        }
        
        $\mathcal{L}_{gvg}=\mathcal{L}_{nce}(g_{2k-1},g_{2k})$\\
        $\mathcal{L}_{gvl}=0.5 * [\mathcal{L}_{nce}(g_{2k},l_{2k-1})] + \mathcal{L}_{nce}(g_{2k-1},l_{2k})]$\\
        $\mathcal{L}_{entropy}=\frac{1}{2N}\sum_{k=1}^{2N} (\|g_i\|_q + \|l_i\|_q)$\\
        $\mathcal{L}=\mathcal{L}_{gvg} + \mathcal{L}_{gvl} + \beta \cdot \mathcal{L}_{entropy}$\\
        update networks $E_{\theta}$
    }
    return encoder network $E_\theta$ 
\end{algorithm}

\begin{algorithm}
  \caption{Testing Pseudocode extension model}
  \label{alg::Testing Pseudocode extension model} 
  \footnotesize
  \KwIn{batch size N, similarity function $sim$, encoder model $E_{\theta_1}$, nonlinear function $ \phi_{\theta_2}$}
    \For{sampled minibatch $\left\{x_k\right\}_{k=1}^N$}
    {
        \For{\textbf{all} $k \in \left\{ 1,....,N\right\}$}
        {
            \# Extract features\\
            $g_{k},l_{k} = E_{\theta_1}(x_k)$\\
            \# nonlinear projection\\
            $\widetilde{l}_{k} = \phi_{\theta_2}(l_{k})$\\
            
        }
        \For{ \textbf{all} $k \in \left\{ 1,....,N\right\}$ }
        {   
            \# Similarity\\
            $s_{gvg_k} = sim(g_k,g_k)$\\
            $s_{gvl_k} = sim(g_k,l_k)$\\
            $s_k = s_{gvg_k} + s_{gvl_k}$\\
            $s'_{k} = c_2\cdot tanh\left(\frac{s_k}{c_1\cdot c_2} \right)$ \\
            \#normal data has larger $s'_i$ \\
        }
    }
    return $roc\_auc\_score(s'_i,\ label_i)$
\end{algorithm}

\end{document}